\documentclass[11pt]{article}
\usepackage[T1]{fontenc}         
\usepackage{amsfonts}      
\usepackage{nicefrac}      
\usepackage{microtype}   
\usepackage[colorlinks,
            linkcolor=red,
            anchorcolor=blue,
            citecolor=blue, 
            pagebackref=true
           ]{hyperref}

\usepackage{fullpage}
\usepackage{tabularx}
\usepackage{float}
\usepackage{enumitem}
\usepackage{wrapfig}
\usepackage{diagbox}
\usepackage{graphicx}   
\usepackage{caption}    
\usepackage{subcaption}

\usepackage{mmll}
\usepackage{tikz}
\usetikzlibrary{arrows,automata,positioning}
\usepackage[colorinlistoftodos, textsize=tiny, textwidth=2.2cm, backgroundcolor=blue!10, linecolor=magenta, bordercolor=magenta]{todonotes}

\definecolor{darkgreen}{rgb}{0.0, 0.6, 0.0}

\renewcommand*{\backref}[1]{}
\renewcommand*{\backrefalt}[4]{%
\ifcase #1 %
  No citations.%
\or
  (p. #2.)%
\else
  (pp. #2.)%
\fi}%
\allowdisplaybreaks

\title{
Rethinking Langevin Thompson Sampling from A Stochastic Approximation Perspective 
}

\author{
    Weixin Wang\thanks{ 
    Duke University; email: {\tt
    weixin.wang@duke.edu}}~\footnotemark[3] 
    ~~
    Haoyang Zheng\thanks{ 
    Purdue University; email: {\tt
    zheng528@purdue.edu}}~\thanks{Equal contribution} 
    ~~
    Guang Lin\thanks{
    Purdue University; email: {\tt
    guanglin@purdue.edu}}
    ~~
    Wei Deng\thanks{
    Morgan Stanley; email: {\tt
    wei.deng@morganstanley.com}}
    ~~
    Pan Xu\thanks{
    Duke University; email: {\tt
    pan.xu@duke.edu}}\\
}

\begin{document}

\maketitle

\begin{abstract}

Langevin-based Thompson Sampling (TS) methods for multi-armed bandits use likelihood gradients to approximate posterior samples when exact sampling is intractable. As new rewards arrive, the posterior changes, and these methods run an inner Markov chain to approximate the updated target. To simplify this repeated sampling procedure, we introduce TS-SA, which incorporates stochastic approximation (SA) within the TS framework. In each round, TS-SA takes one Langevin step using the newly observed reward and applies an SA update to combine the noisy proposal with the current parameter estimate. By retaining this estimate across rounds, TS-SA maintains one continuing Markov chain for each arm, preserving the influence of past observations without storing or sampling them again. Within this recursion, the proposal step size remains fixed across rounds, while the SA weights decrease as observations accumulate to control the effect of noise on the estimate. This continuing-chain formulation simplifies the theoretical analysis, allowing us to establish a near-optimal worst-case regret bound of $\widetilde{\mathcal O}(\sqrt{KT})$ under log-concavity and smoothness assumptions with one update per new reward. Our experiments further demonstrate that TS-SA, using a single-step Langevin update after an initial warm-up, achieves strong performance against standard baselines.

\end{abstract}

\section{Introduction}
\label{sec:intro}

The multi-armed bandit (MAB) problem \citep{lai1987adaptive,slivkins2019introduction,lattimore2020bandit} formalizes the fundamental exploration-exploitation trade-off in sequential decision-making, where an agent repeatedly selects actions under uncertainty to maximize cumulative rewards. One prominent approach to address this trade-off is Thompson Sampling (TS) \citep{thompson1933likelihood}, which maintains a posterior over reward models and chooses actions by sampling from that posterior. However, exact posterior sampling is often intractable without closed-form expressions. Existing work has established regret guarantees for TS with tractable sampling distributions \citep{agrawal2017near,jin2021mots}, and some methods extend to exponential family models \citep{korda2013thompson,jin2022finite,jin2023thompson}. For other reward models, computing posterior samples remains challenging, particularly when the prior and likelihood are not conjugate.

To overcome these computational difficulties, recent studies have explored approximate posterior sampling techniques, including variational inference \citep{clavier2023vits}. Notably, \citet{mazumdar2020thompson} proposed Thompson Sampling with Stochastic Gradient Langevin Dynamics (TS-SGLD), an approximate sampling extension of TS for log-concave reward families. By employing SGLD within each round, they obtain approximate posterior samples through Langevin Monte Carlo (LMC) updates and establish rigorous theoretical guarantees. However, TS-SGLD faces several challenges: (1) each new reward changes the posterior target, so a separate inner Markov chain is run to approximate the updated posterior; (2) tracking these changing posteriors requires a round-dependent step size for the inner sampling updates, complicating sampling procedure and practical tuning; (3) controlling approximation errors across changing targets complicates the theoretical analysis; and (4) its stochastic gradients use minibatches from the observed reward history, requiring access to past rewards and increasing storage requirements as observations accumulate.

\begin{figure*}
\centering
    \includegraphics[width=\textwidth]{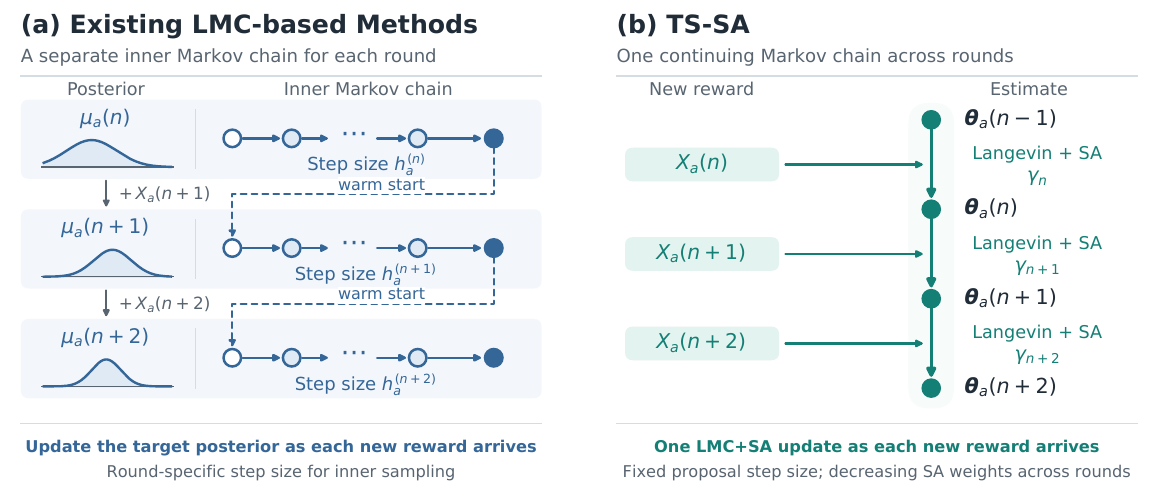}
\caption{Comparison of updates for one arm, where $n$ is the number of pulls. Left (existing LMC-based methods): each reward changes the posterior. An inner Markov chain, initialized from the previous output, approximates the updated posterior using a step size $h_a^{(n)}$ that depends on the number of observations. Right (TS-SA): each reward advances one continuing Markov chain through a Langevin proposal and an SA update.}
\label{fig:comparison}
\vspace{-.10 in}
\end{figure*}

To address these challenges, we propose Thompson Sampling with Stochastic Approximation (TS-SA), which combines Langevin-based updates \citep{mazumdar2020thompson} with SA to maintain one continuing Markov chain for each arm. In each round, we form a Langevin proposal using the newly observed reward and combine it with the current parameter estimate through an SA update. The estimate preserves the influence of past observations without storing or sampling them again. The proposal step size $h$ remains fixed across rounds, while the SA weights decrease as observations accumulate to control the effect of noise. \Cref{fig:comparison} illustrates the update procedure difference between the existing LMC-based methods and TS-SA.

Although each update uses a different reward, the expected update follows the gradient of the same population objective, whose maximizer is the true reward parameter. This objective also defines a stationary reference distribution for the population Langevin dynamics, with its mode at the true parameter. We directly analyze how the SA estimates concentrate around this mode. The continuing-chain formulation simplifies the theoretical analysis and achieves near-optimal regret.

\textbf{Our contributions} are summarized as follows:

\begin{itemize}[leftmargin=10pt,nosep]
    \item \textbf{Algorithm design:} We introduce TS-SA, which maintains one continuing Markov chain for each arm through a single Langevin proposal and an SA update per new reward. Each parameter update requires one likelihood-gradient evaluation and no repeated access to historical rewards, giving a memory cost independent of the time horizon.

    \item \textbf{Theoretical guarantees:} We establish a near-optimal worst-case regret bound of $\widetilde{\mathcal O}(\sqrt{KT})$ for TS-SA under log-concavity and smoothness assumptions , where $K,T$ are the number of arms and rounds, respectively. The continuing-chain formulation allows us to use stochastic approximation analysis to control estimation errors through update contraction, simplifying the concentration proof.

    \item \textbf{Empirical validation:} Empirical evaluations across diverse simulation settings further validate that TS-SA, using a single-step LMC update with reasonable warm-up, can outperform existing MAB algorithms substantially.
\end{itemize}

\section{Related Work}

\paragraph{Randomized Exploration.} In MABs, TS \citep{agrawal2017near} is a pivotal randomized exploration strategy, achieving near-optimal regret bounds of $\mathcal{O}(\sqrt{K T \ln T})$ for Beta priors and $\mathcal{O}(\sqrt{K T \ln K})$ for Gaussian priors, closely matching the theoretical lower bound of $\Omega(\sqrt{K T})$, which can be achieved by variants of TS such as the MOTS \citep{jin2021mots} and $\epsilon$-TS \citep{jin2023thompson} algorithms. Unlike Upper Confidence Bound (UCB) algorithms, which rely on deterministic confidence intervals \citep{lattimore2020bandit}, TS samples from posterior distributions, enabling flexible prior integration and robust exploration \citep{thompson1933likelihood,chapelle2011empirical}. However, exact posterior sampling in TS is computationally intensive, particularly for non-conjugate priors, and its theoretical guarantees may falter in complex environments \citep{russo2019worst}. To address this, approximate sampling methods such as LMC \citep{xu2022langevin,karbasi2023langevin}, Metropolis-Hastings \citep{huix2023tight}, variational inference \citep{clavier2023vits}, and SGLD variants \citep{mazumdar2020thompson,zheng2024accelerating} have been developed to preserve competitive regret bounds for problem-dependent settings, which are further extended to reinforcement learning problems \citep{ishfaq2023provable,ishfaq2024more,hsu2024randomized}. Perturb History Exploration (PHE) further enhances efficiency by perturbing historical data to approximate posterior sampling, offering applicability to complex reward distributions \citep{kveton2019randomized}. Ensemble sampling maintains a small set of independently perturbed model replicas and selects each action according to one randomly chosen replica \citep{lu2017ensemble,qin2022analysis, janz2024ensemble, lee2024improved}. For approximate posterior sampling, controlling approximation errors is essential, since these errors can lead to under- or over-exploration and degrade regret performance \citep{phan2019thompson}.

\paragraph{Stochastic Approximation.} Stochastic approximation provides a standard framework for analyzing adaptive sampling algorithms that iteratively alternate between sampling and averaging \citep{robbins1951stochastic}. Following this, \citet{benveniste2012adaptive} systematically developed adaptive algorithms to average out the Markovian noise; \citet{borkar2000ode} offered rigorous convergence analysis from the perspective of ordinary differential equations; \citet{borkar2008stochastic} presented intuitive stability criteria based on dynamical systems theory and Lyapunov functions. These theoretical advances have led to SA’s broad adoption across numerous domains, including MCMC~\citep{liang2007stochastic,levin2017markov}, dynamic programming~\citep{bertsekas1996neuro,haddad2008nonlinear}, reinforcement learning~\citep{tsitsiklis1994asynchronous,srikant2019finite,borkar2000ode,durmus2024finite}, optimization~\citep{kushner2003stochastic,spall2005introduction,kushner2012stochastic,lan2020first}. In bandit problems, \citet{ding2021efficient} combine online stochastic gradient updates with TS to obtain efficient parameter estimates for generalized linear bandits. Building on this framework, TS-SA combines each Langevin proposal with the current parameter estimate through an SA update. Decreasing SA weights control the effect of reward and Gaussian noise while preserving information from earlier observations in the current estimate. This continuing recursion allows us to apply stochastic approximation analysis directly, using update contraction to derive the concentration bounds needed for the regret guarantee.

\section{Preliminaries}
\label{sec:prelim}

\paragraph{Multi-armed Bandit.}

We consider a MAB problem with $K\geq2$ arms, indexed by $\mathcal A=\{1,2,\ldots,K\}$. Each arm $a$ has a reward distribution $p_a(\cdot|\btheta_a^*)$, where $\btheta_a^*\in\mathbb R^d$ is fixed and unknown. The rewards from each arm are i.i.d. and independent across different arms. For each arm, the known feature $\bphi_a\in\mathbb R^d$ satisfies $\bar X_a=\EE_{X\sim p_a(\cdot|\btheta_a^*)}[X]=\bphi_a^\top\btheta_a^*$, and we write $B_a=\|\bphi_a\|_2>0$. At round $t$, the agent selects an arm $A_t$ and receives a new reward $X_t$. Assume that arm $1$ is uniquely optimal, and let $\Delta_a=\bar X_1-\bar X_a>0$ for $a>1$. The goal is to minimize the expected cumulative pseudo-regret $\EE[\mathcal R(T)]$, where $\mathcal R(T)=\sum_{t=1}^T(\bar X_1-\bar X_{A_t})$. Let $\mathcal T_a(t)$ be the number of pulls of arm $a$ before round $t$, and let $X_a(n)$ denote its $n$-th reward.

\paragraph{Thompson Sampling.}

To balance exploration and exploitation, TS maintains a posterior distribution over the reward parameter of each arm and selects arms using samples from these distributions. The TS framework is shown in \Cref{alg:thompson_main}. Given a prior density $\pi_a$ and $n$ observed rewards from arm $a$, the posterior density is $\mu_a(n)(\btheta)\propto\pi_a(\btheta)\prod_{i=1}^n p_a(X_a(i)|\btheta)$, with $\mu_a(0)=\pi_a$. We also write $\mu_a(n)[\tau_a]$ for the scaled posterior with density proportional to $\mu_a(n)(\btheta)^{\tau_a}$, where $\tau_a>0$ controls its spread. Setting $\tau_a=1$ for all arms recovers standard TS.

\begin{algorithm}[htbp]
\caption{Thompson Sampling in MAB}
\label{alg:thompson_main}
\begin{algorithmic}[1]
  \REQUIRE Bandit features $\bphi_a$, priors $\pi_a$, and scaling parameters $\tau_a>0$ for $a\in\mathcal A$.
  \STATE For all $a\in\mathcal A$, set $\mu_a(0)=\pi_a$.
  \FOR{$t=1,2,\ldots,T$}
    \STATE Set $n_a=\mathcal T_a(t)$ for all $a\in\mathcal A$.
    \STATE Sample $\btheta_{a,t}\sim\mu_a(n_a)[\tau_a]$ for all $a\in\mathcal A$.
    \STATE Choose $A_t=\argmax_{a\in\mathcal A}\bphi_a^\top\btheta_{a,t}$ and receive $X_t$.
    \STATE Record $X_{A_t}(n_{A_t}+1)=X_t$ and update the selected arm's posterior to $\mu_{A_t}(n_{A_t}+1)$.
  \ENDFOR
\end{algorithmic}
\end{algorithm}

\paragraph{Thompson Sampling with SGLD.}

To implement TS when exact posterior sampling is intractable, recent works \citep{mazumdar2020thompson,xu2022langevin,ishfaq2023provable} use Langevin-based methods to generate approximate samples. These approaches exploit the structure of log-concave posteriors and use iterative updates that combine gradients of the log-likelihood with injected noise. In particular, \citet{mazumdar2020thompson} developed TS-SGLD with regret guarantees under log-concavity and smoothness assumptions. To better understand its mechanism, next we introduce the LMC update rule used in TS-SGLD.

Fix an arm and let $\btheta^*\in\mathbb R^d$ denote its true reward parameter. Given $n\geq1$ rewards $\{X_i\}_{i=1}^n\overset{\mathrm{i.i.d.}}{\sim}p(\cdot|\btheta^*)$ and a prior density $\pi$, we aim to approximate the scaled posterior $\mu(n)[\tau](\btheta)\propto\exp(\tau[nF_n(\btheta)+\log\pi(\btheta)])$, where $F_n(\btheta)=\frac1n\sum_{i=1}^n\log p(X_i|\btheta)$ and $\tau>0$ is an inverse temperature parameter. The corresponding Langevin diffusion is given by
\begin{align*}
    d\btheta_s=\left[\nabla_{\btheta}F_n(\btheta_s)+\frac1n\nabla_{\btheta}\log\pi(\btheta_s)\right]ds+\sqrt{\frac{2}{n\tau}}\,dB_s,
\end{align*}
where $B_s$ is a standard $d$-dimensional Brownian motion. For fixed observed data and under appropriate conditions, the distribution of $\btheta_s$ converges to $\mu(n)[\tau]$ as $s\to\infty$. In practice, we approximate the Langevin diffusion above using the Euler--Maruyama discretization. For a suitable step size $h>0$, the update at iteration $j$ becomes
\begin{align}\label{eq:euler_update}
    \btheta^{(j+1)} =\btheta^{(j)}+h\Big[\nabla_{\btheta}F_n(\btheta^{(j)})+  \frac1n\nabla_{\btheta}\log\pi(\btheta^{(j)})\Big] + \frac1{\sqrt{n\tau}}\mathcal N(\zero,2h\Ib).
\end{align}
The iteration count $N$ and step size $h$ determine approximation accuracy and computational cost. SGLD estimates $\nabla_{\btheta}F_n(\btheta^{(j)})$ from a uniformly sampled minibatch of past rewards and evaluates the prior gradient directly. Each new reward changes the posterior, so TS-SGLD runs an inner sampling phase for the updated target, starting from the previous sampling state. This initialization allows $N$ and the minibatch size to remain constant as observations accumulate \citep{mazumdar2020thompson}. The analysis must still control approximation errors across changing targets, and gradient computation requires access to the reward history.

\section{Thompson Sampling with Stochastic Approximation}
\label{sec:algorithm}

\subsection{Algorithm Design}

\paragraph{Design Intuition.}
We adopt the sampling-based exploration framework of TS \citep{agrawal2017near} and use SA to update each arm's parameter estimate. To simplify repeated posterior approximation in Langevin-based methods \citep{mazumdar2020thompson}, TS-SA maintains one continuing Markov chain for each arm. Each new reward is used to form a Langevin proposal, which is combined with the current estimate through an SA update. The estimate is retained between pulls, while the proposal step size $h$ remains fixed and the SA weights decrease as observations accumulate across rounds.

The expected updates follow a fixed population objective whose maximizer is the true reward parameter. This allows us to analyze the continuing recursion using stochastic approximation techniques and derive the near-optimal regret guarantees in \Cref{sec:theortical_analysis}.

\begin{algorithm}
\caption{Thompson Sampling with Stochastic Approximation (TS-SA)}
\label{alg:TS_SA}
\small
\begin{algorithmic}[1]
  \REQUIRE Bandit features $\bphi_a\in\mathbb R^d$, priors $\pi_a$ for $a\in\mathcal A$, horizon $T\geq K$, and parameters $h,\{\gamma_n\}_{n\geq1},\{\tau_a\}_{a\in\mathcal A}$ chosen according to (\ref{equ:sa_stepsizes}--\ref{equ:sa_temperature_explicit}).

  \STATE \textbf{Initialization:} Sample $\btheta_a(0)\sim\pi_a$ for all $a\in\mathcal A$.

  \FOR{$t=1,\ldots,T$}
    \STATE Set $n_a=\sum_{s=1}^{t-1}\mathbf{1}\{A_s=a\}$ for all $a\in\mathcal A$. \label{line:arm_choose_begin}

    \IF{$t\leq K$}
      \STATE Set $A_t=t$.
    \ELSE
      \STATE Sample $\btheta_{a,t}\sim\mathcal N\bigl(\btheta_a(n_a), \frac{1}{\tau_a n_a}\Ib\bigr)$  for all $a\in\mathcal A$. \label{line:scaled_sample}

      \STATE Choose $A_t=\argmax_{a\in\mathcal A}\bphi_a^\top\btheta_{a,t}$.
    \ENDIF

    \STATE Pull arm $A_t$ and receive reward $X_t$. \label{line:arm_choose_end}

    \STATE $\btheta\leftarrow\btheta_{A_t}(n_{A_t})$. \label{line:parameter_update_begin}

    \STATE $\bomega\leftarrow\btheta+h\nabla_{\btheta}\log p_{A_t}(X_t\mid\btheta)+\mathcal N(\zero,2h\Ib)$. \label{line:one_step_SGLD}

    \STATE $\btheta_{A_t}(n_{A_t}+1)=(1-\gamma_{n_{A_t}+1})\btheta+\gamma_{n_{A_t}+1}\bomega$. \label{line:SA}\label{line:parameter_update_end}
  \ENDFOR
\end{algorithmic}
\end{algorithm}

\paragraph{Algorithm Interpretation.}
Each round of \Cref{alg:TS_SA} consists of arm selection and a parameter update. In the first stage (Lines \ref{line:arm_choose_begin}--\ref{line:arm_choose_end}), we pull each arm once during the first $K$ rounds. In subsequent rounds, we sample $\btheta_{a,t}$ from a Gaussian distribution centered at $\btheta_a(n_a)$ for each arm $a$ and select the arm with the largest sampled mean reward $\bphi_a^\top\btheta_{a,t}$. As we collect more rewards from arm $a$, the sampling distribution becomes more concentrated around its current estimate. The parameter $\tau_a$ controls the amount of Gaussian noise used for exploration.

The second stage (Lines \ref{line:parameter_update_begin}--\ref{line:parameter_update_end}) updates only the selected arm $A_t$. Starting from its current estimate, we form a Langevin proposal $\bomega$ using the new reward $X_t$ (Line \ref{line:one_step_SGLD}). The SA step then takes a weighted average of the current estimate and this proposal, assigning weight $\gamma_{n_{A_t}+1}$ to the proposal (Line \ref{line:SA}). Each parameter update uses the newly observed reward and requires only one evaluation of the log-likelihood gradient.

Combining the proposal and SA steps, the estimate after the $n$th pull of arm $a$ satisfies
\begin{align*}
    \btheta_a(n)=\btheta_a(n-1)+h_n'\nabla_{\btheta}\log p_a(X_a(n)|\btheta_a(n-1))+h_n'\sqrt{2/h}\,\boldsymbol{\xi}_{a,n},
\end{align*}
where $X_a(n)$ denotes the reward from the $n$th pull of arm $a$, $h_n'=h\gamma_n$, and $\boldsymbol{\xi}_{a,n}\sim\mathcal N(\zero,\Ib)$. This expression shows that the SA step scales both the likelihood gradient and the Gaussian noise. We keep the proposal step size $h$ fixed and choose $\gamma_n$ so that the effective step size $h_n'$ decreases at rate $1/n$. The parameter choices and notation are given in \Cref{sec:notations}.

\subsection{A Fixed Population Target}
\label{sec:stationary_target}

For each arm, we interpret parameter estimation as minimizing the KL divergence from the true reward distribution to the candidate model. We therefore define the population objective
\begin{align}\label{equ:population_objective_main}
    F_a(\btheta)=\EE_{X\sim p_a(\cdot|\btheta_a^*)}\left[\log\frac{p_a(X|\btheta)}{p_a(X|\btheta_a^*)}\right].
\end{align}
This is the negative KL divergence, so maximizing $F_a$ minimizes the divergence between the two distributions. In particular, $F_a(\btheta)\leq0$ and $F_a(\btheta_a^*)=0$. Under strong log-concavity of the likelihood in $\btheta$, $F_a$ is strongly concave and uniquely maximized at $\btheta_a^*$. Since the true reward distribution is fixed, this objective remains unchanged as observations accumulate. Although $F_a$ involves the unknown true distribution, its gradient can be estimated using a new reward:
\begin{align*}
    \nabla_{\btheta}F_a(\btheta)=\EE_{X\sim p_a(\cdot|\btheta_a^*)}[\nabla_{\btheta}\log p_a(X|\btheta)].
\end{align*}
Each reward is independent of previous updates, so its likelihood gradient is unbiased at the current estimate. With the zero-mean Gaussian noise, this gives
\begin{align*}
    \EE[\btheta_a(n)\mid\btheta_a(n-1)=\btheta]=\btheta+h_n'\nabla_{\btheta}F_a(\btheta).
\end{align*}
Thus, the expected update follows the gradient of the negative KL divergence. The SA recursion accumulates these noisy updates to estimate the true parameter, retaining information from earlier observations through the current estimate.

\paragraph{Stationary Reference Distribution.}
The population objective also explains the Langevin form of the proposal. Define
\begin{align}\label{equ:population_reference_main}
    \mu_a^{\mathrm{ref}}(d\btheta)=Z_a^{-1}e^{F_a(\btheta)}\,d\btheta,
\end{align}
where $Z_a$ is finite by strong concavity. Since $F_a$ is the negative KL divergence, smaller divergence gives higher density. Its log-density gradient is $\nabla_{\btheta}F_a$, making $\mu_a^{\mathrm{ref}}$ stationary for the auxiliary Langevin diffusion
\begin{align*}
    d\boldsymbol\vartheta_s=\nabla_{\btheta}F_a(\boldsymbol\vartheta_s)\,ds+\sqrt2\,dB_s,
\end{align*}
where $B_s$ is a standard $d$-dimensional Brownian motion. Discretizing this diffusion with step size $h$ and using the new reward's likelihood gradient gives the proposal in Line \ref{line:one_step_SGLD}.

This reference distribution provides a sampling interpretation for the Langevin proposal. The subsequent SA step scales the proposal displacement by $\gamma_n$, giving Gaussian noise with covariance $2h\gamma_n^2\Ib$ in the updated estimate. As observations accumulate, the decreasing weights reduce the contribution of this noise. The SA recursion uses these noisy proposals to estimate $\btheta_a^*$, the mode of the reference distribution.

\subsection{Comparison with TS-SGLD}
\label{sec:comparison_with_ts-sgld}

\paragraph{Markov Chain Structure.}

After each posterior update, TS-SGLD runs an inner Markov chain using minibatches of past rewards and the previous output as its starting point. In \citet{mazumdar2020thompson}, its step size changes with the number of observations, while this starting point allows the number of inner updates and batch size to stay fixed. TS-SA keeps one continuing SA chain, using each new reward for an update with a fixed proposal step size $h$ and decreasing SA weights $\gamma_n$. Earlier rewards influence each new update through the current estimate, so they need not be stored or reused. The update weights change, but the population objective stays the same.

\paragraph{Convergence Analysis.}

The two methods control different errors in their convergence analysis. For TS-SGLD, the proof bounds the sampling error for each posterior and shows that the output gives a suitable starting point for the next round. For TS-SA, we track how the estimation error changes as new rewards arrive. After $n$ pulls, the estimate has been updated $n$ times with new rewards. We analyze how these updates improve the estimate, without a separate sampling error bound for each posterior.

\paragraph{Memory Cost.}

TS-SGLD computes stochastic gradients from minibatches drawn from each arm's accumulated reward history. Storing this history, together with the current sampling states and pull counts, requires $\mathcal O(T+Kd)$ memory after $T$ rounds. TS-SA uses only the current estimate and the new reward for each update. Earlier rewards continue to influence the estimate through its current value and need not be stored or sampled again. Maintaining the estimates and pull counts therefore requires $\mathcal O(Kd)$ memory, independent of the time horizon.

\section{Theoretical Analysis}
\label{sec:theortical_analysis}

\subsection{Concentration and Convergence of TS-SA}
\label{sec:posterior_concentration_main}

We analyze the parameter estimates generated by the continuing SA recursion of TS-SA. We establish their convergence and concentration and use these results to derive the regret bound. The analysis relies on the following assumptions on the reward families.

\begin{assumption}[Strong log-concavity for $\btheta_a$]
\label{assum:strongly_concave_theta}
For any $a\in\mathcal A$, $X\in\mathbb R$, and $\btheta_a,\btheta_a^\prime\in\mathbb R^d$, we have $\langle\nabla_{\btheta}\log p_a(X|\btheta_a)-\nabla_{\btheta}\log p_a(X|\btheta_a^\prime),\btheta_a-\btheta_a^\prime\rangle\leq-m_a\|\btheta_a-\btheta_a^\prime\|_2^2$, where $m_a>0$.
\end{assumption}

\begin{assumption}[Strong log-concavity for $X$]
\label{assum:strongly_concave_X}
For any $a\in\mathcal A$ and $\btheta_a\in\mathbb R^d$, the reward density is positive on $\mathbb R$, $\log p_a(\cdot|\btheta_a)\in C^2(\mathbb R)$, and $\nabla_X^2\log p_a(X|\btheta_a)\leq-\nu_a$ for all $X\in\mathbb R$, where $\nu_a>0$.
\end{assumption}

\begin{assumption}[Lipschitz for $\nabla_{\btheta}\log p_a(X|\btheta_a)$]
\label{assum:joint_lipschitz}
For any $a\in\mathcal A$, $X,X^\prime\in\mathbb R$, and $\btheta_a,\btheta_a^\prime\in\mathbb R^d$, we have $\|\nabla_{\btheta}\log p_a(X|\btheta_a)-\nabla_{\btheta}\log p_a(X^\prime|\btheta_a^\prime)\|_2\leq L_a(|X-X^\prime|+\|\btheta_a-\btheta_a^\prime\|_2)$, where $L_a>0$.
\end{assumption}

These assumptions control update contraction and fluctuations caused by new rewards. We use \Cref{assum:strongly_concave_X} only at $\btheta_a^*$. Independence and Gaussian initialization are specified in \Cref{assum:score_initialization}, and parameter choices and constants are given in \Cref{sec:notations}.

We analyze the continuing recursion using stochastic approximation techniques. Decreasing SA weights control the noise introduced by new updates, while contraction reduces the effect of earlier perturbations. We use these properties to bound how each reward and Gaussian input affects later estimates. Combining the resulting concentration bound with the mean-square error bound gives the exponential moment bound below. The complete proofs are provided in \Cref{appendix:sec:posterior:concentration}.

\paragraph{Mean-square Convergence.}
We first bound the mean-square error of $\btheta_a(n_a)$, establishing its convergence to the true parameter $\btheta_a^*$.

\begin{lemma}[Mean-square convergence of TS-SA]
\label{lem:posterior_concentration_main}
Suppose \Crefrange{assum:strongly_concave_theta}{assum:joint_lipschitz} and \Cref{assum:score_initialization} hold. Under the step-size conditions in \eqref{equ:sa_stepsizes} and \eqref{equ:sa_schedule}, \Cref{alg:TS_SA} satisfies, for every arm $a$,
\begin{align*}
    \EE[\|\btheta_a(n_a)-\btheta_a^*\|_2^2]\leq\frac{U_a}{n_a+b},
\end{align*}
where  $n_a\geq0$ is fixed, $U_a$ is defined in \eqref{equ:sa_constants}. 
\end{lemma}

\paragraph{Convergence of TS-SA.}

The mean-square bound controls the average squared estimation error. We next derive an exponential tail bound for the estimation error, which yields convergence in Wasserstein distance.

\begin{lemma}[Convergence of TS-SA]
\label{lemma:discrete:sa}
Under the conditions of \Cref{lem:posterior_concentration_main}, for every arm $a$ and fixed $n_a\geq1$, we have that
\begin{align*}
    \EE\exp\left(\frac{n_a\|\btheta_a(n_a)-\btheta_a^*\|_2^2}{S_a^2}\right)\leq2.
\end{align*}
Consequently, for every $p\geq1$,
\begin{align*}
    \mathcal W_p\big(\widehat\mu_a^{\text{SA}}(n_a),\delta_{\btheta_a^*}\big)=\EE[\|\btheta_a(n_a)-\btheta_a^*\|_2^p]^{1/p}\leq S_a\sqrt{\frac{p}{n_a}},
\end{align*}
where $\widehat\mu_a^{\text{SA}}(n_a)$ is the distribution of $\btheta_a(n_a)$, $\delta_{\btheta_a^*}$ is the point mass at $\btheta_a^*$, and $\mathcal W_p$ is the Wasserstein-$p$ distance. The constant $S_a$ is defined in \eqref{equ:sa_constants} and is independent of $n_a$, $T$, and $\tau_a$.
\end{lemma}

\begin{remark}
\Cref{lemma:discrete:sa} shows that one LMC--SA update per new reward is sufficient for an $\mathcal O(1/\sqrt{n_a})$ convergence rate in Wasserstein-$p$ distance for every $p\geq1$. The proposal step size $h$ remains fixed, while the effective step size decreases as $h\gamma_{n_a}=\mathcal O(1/n_a)$.
\end{remark}

\paragraph{Concentration of TS-SA Samples.}
For fixed $n_a\geq1$, let $\widehat\mu_a^{\text{SA}}(n_a)[\tau_a]$ be the distribution of $\btheta_a(n_a)$ plus independent $\mathcal N(\zero,(\tau_a n_a)^{-1}\Ib)$ noise. Combining the preceding estimate bound with Gaussian concentration gives the following result.

\begin{theorem}[Concentration of TS-SA decision samples]
\label{theorem:posterior_concentration_main}
Under the conditions of \Cref{lem:posterior_concentration_main}, for every arm $a$, fixed $n_a\geq1$, and $\delta\in(0,1)$,
\begin{align*}
    \PP_{\btheta\sim\widehat\mu_a^{\text{SA}}(n_a)[\tau_a]}\left(\|\btheta-\btheta_a^*\|_2>\sqrt{\frac{C_1}{n_a}\log\frac4\delta}\right)\leq\delta,
\end{align*}
where $C_1=\max_{a\in\mathcal A}(S_a+(\sqrt d+\sqrt2)/\sqrt{\tau_a})^2$ is independent of $n_a$, $\delta$, and $T$.
\end{theorem}

\begin{remark}
\Cref{theorem:posterior_concentration_main} shows that the samples used for arm selection concentrate around the true parameters at rate $\mathcal O(1/\sqrt{n_a})$. In the regret analysis, concentration controls overly optimistic samples from suboptimal arms, while an anti-concentration bound ensures sufficient exploration of the optimal arm. Our analysis obtains these guarantees directly from the continuing SA recursion, without requiring a separate posterior approximation bound at each round.
\end{remark}

\subsection{Regret Bound}

With suitable exploration parameters $\tau_a$, the preceding bounds yield the following regret guarantee. The complete proof is provided in \Cref{sec:regret_analysis}.

\begin{theorem}[Regret bound]
\label{theorem:regret_main}
Under the conditions of \Cref{lemma:discrete:sa}, choose $0<\tau_a\leq1/(2S_a^2)$ for every arm $a$. Let $v_a=B_a^2/\tau_a$ and $C_0=1+2e\sqrt{2\pi}$. For every horizon $T\geq K$, including the $K$ initialization pulls, \Cref{alg:TS_SA} satisfies
\begin{align*}
    \EE[\mathcal R(T)]\leq\sum_{a>1}\left[3\Delta_a+\frac{32}{\Delta_a}\big(v_a\log(eT)+C_0v_1\big)\right].
\end{align*}
Moreover, with $v_\star=\max_a v_a$, $\Delta_{\max}=\max_{a>1}\Delta_a$, and $C_\star=32(1+C_0)$,
\begin{align*}
    \EE[\mathcal R(T)]\leq2\sqrt{C_\star v_\star KT\log(eT)}+3K\Delta_{\max}.
\end{align*}
\end{theorem}

\begin{remark}
The condition on $\tau_a$ keeps the sampling noise sufficiently large relative to the estimation error. Together with the exponential moment bound in \Cref{lemma:discrete:sa}, it controls the regret caused by underestimation of the optimal arm, as shown in \Cref{lem:anti_concentration}. The theorem gives logarithmic instance-dependent regret and worst-case bound of order $\widetilde{\mathcal O}(\sqrt{KT})$.
\end{remark}

\section{Experiments}\label{sec:experiment}

In this section, we evaluate the regret performance of TS-SA with one parameter update per round across different bandit settings. We first compare it with MAB baselines under different reward distributions, reward gaps, and numbers of arms. To understand which implementation choices matter in practice, we further examine the effects of taking additional updates, changing the SA step-size schedule, and using multiple recent observations to estimate gradients.

\paragraph{Baselines.}

We compare our proposed algorithm with several well-known MAB algorithms, including UCB \citep{auer2002finite}, TS \citep{agrawal2017near}, $\epsilon$-TS \citep{jin2023thompson}, and TS-SGLD \citep{mazumdar2020thompson}. All these methods select an arm at round $t$ according to $A_t=\arg\max_a\langle\bphi_a,\btheta_{a,t}\rangle$, where the construction of $\btheta_{a,t}$ for each method is specified in \Cref{appendix:subsec:baseline}.

\begin{remark}
In the main experiments, TS-SA performs one LMC update followed by one SA update per round, as in \Cref{alg:TS_SA}. For gradient estimation, we use the most recent rewards. Specifically, let $X_a(n)$ denote the latest reward included in the gradient and $m=\min\{n,\mathcal B\}$. At the current estimate $\btheta$, we compute $\boldsymbol g=\frac{1}{m}\sum_{i=0}^{m-1}\nabla_{\btheta}\log p_a(X_a(n-i)|\btheta)$ and form $\bomega=\btheta+h\boldsymbol g+\mathcal N(\zero,2h\Ib)$, replacing Line \ref{line:one_step_SGLD} in \Cref{alg:TS_SA}. We then apply the SA update in Line \ref{line:SA}. Our theoretical analysis considers the single-reward update, while the recent-reward window is used as an experimental extension.
\end{remark}

\paragraph{Implementation and Experimental Setup.} 

We evaluate our method in two reward settings: Single Gaussian Reward (SGR) and Mixture Gaussian Rewards (MGR). For the main experiments, we consider $K\in\{10,50\}$ arms and run each experiment for $T=10,000$ rounds, with results averaged over 50 independent trials. To ensure a fair comparison, we standardize the hyperparameter tuning process for all methods under evaluation. Specifically, for TS-SA, we tune the proposal step size $h$ and SA weight $\gamma$. For UCB, we tune the temperature $\tau$ to estimate $\hat{\mu}_a(t) + \tau \sqrt{2 \log t / \mathcal{T}_a(t)}$. For TS, we tune the temperature $\tau$ in the posterior variance $\tau / (1/\sigma_0^2 + \mathcal{T}_a(t))$. For $\epsilon$-TS, we tune both the temperature $\tau$ and the exploration probability $\epsilon$. For TS-SGLD, we tune the temperature $\tau$ and the SGLD step size $h$. In SGR, we construct a problem where the optimal arm has a mean reward of $\mu_1 = 3.0$, and the remaining $K-1$ arms have identical mean rewards of $\mu_i = 3.0 - \Delta$ for $i=2, \dots, K$. We test two reward gaps: $\Delta \in \{0.1, 0.5\}$, and all arms are endowed with Gaussian noise of variance $\sigma^2=1.0$. To evaluate performance under a more complex reward distribution, we further design a multi-arm bandit problem for a mixture Gaussian setting. Here, the optimal arm's reward is drawn from a uniform mixture of two Gaussians: $\frac{1}{2} \mathcal{N}(\mu_1, \sigma^2) + \frac{1}{2} \mathcal{N}(\mu_1+\Delta, \sigma^2)$ where $\Delta \in \{0.1, 0.5\}$. Conversely, the rewards for the suboptimal arms are drawn from a similar mixture with lower means: $\frac{1}{2} \mathcal{N}(\mu_1-\Delta, \sigma^2) + \frac{1}{2} \mathcal{N}(\mu_1, \sigma^2)$.  

We provide some implementation details for TS-SA. To update the selected arm's estimate, we compute the gradient using the most recent $\mathcal B$ rewards and perform one LMC--SA update per round. We also include a warm-up phase by pulling each arm $\Omega$ times before applying TS-SA. This provides initial observations from every arm before the algorithm begins selecting arms based on its current estimates. The empirical SA weight schedule and additional implementation details are provided in \Cref{sec:ablation}.

\begin{figure*}[!t]
  \centering
        \begin{subfigure}[b]{0.24\textwidth}
            \centering
            \includegraphics[width=\textwidth]{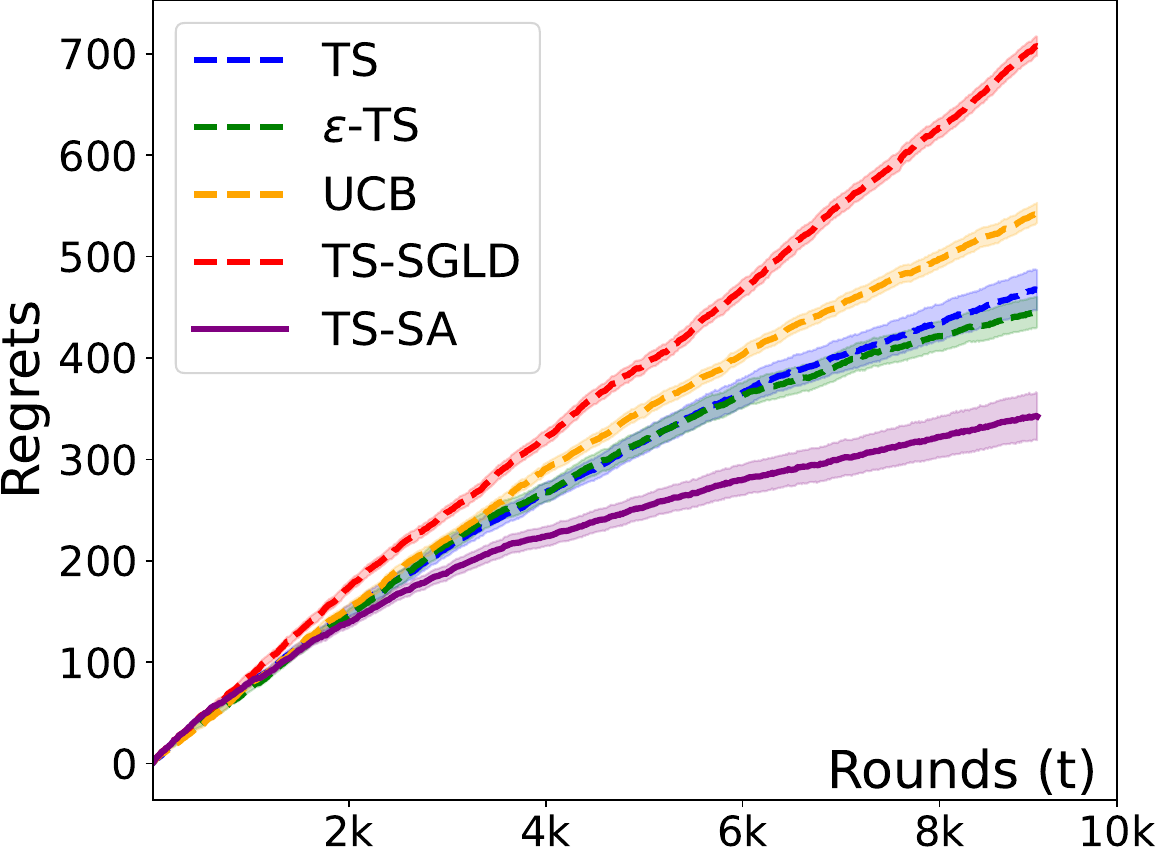}
        \caption{SGR,$\Delta = 0.1$,$K=10$}
        \end{subfigure}
        \begin{subfigure}[b]{0.24\textwidth}
            \centering
            \includegraphics[width=\textwidth]{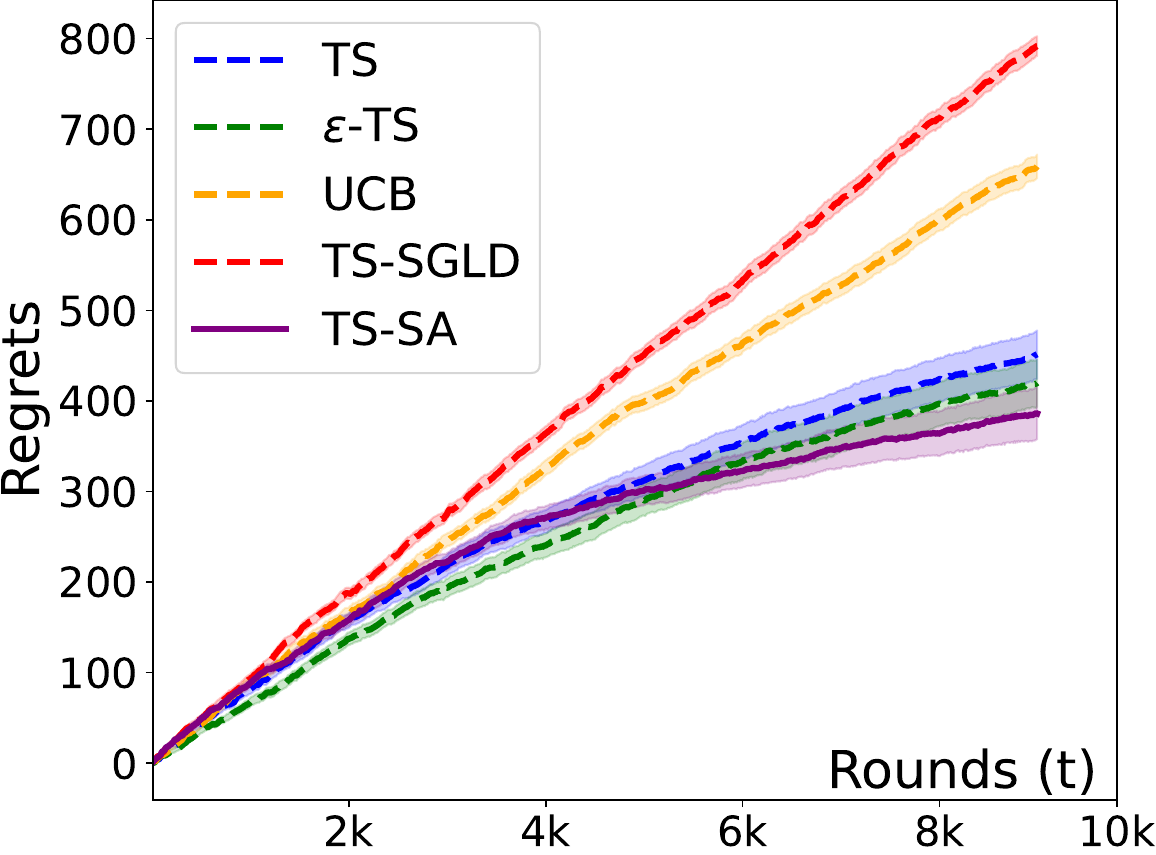}
        \caption{MGR,$\Delta = 0.1$,$K=10$}
        \end{subfigure}
        \begin{subfigure}[b]{0.24\textwidth}
            \centering
            \includegraphics[width=\textwidth]{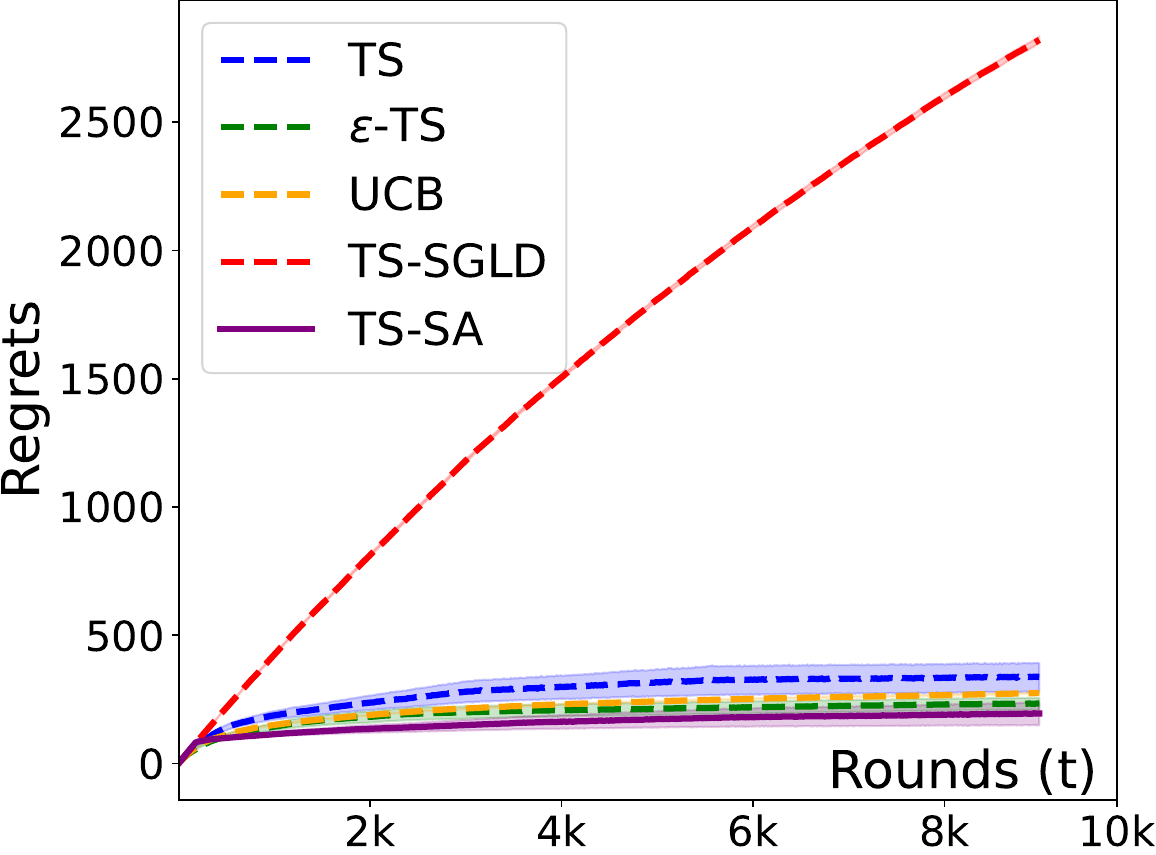}
        \caption{SGR,$\Delta = 0.5$,$K=10$}
        \end{subfigure}
        \begin{subfigure}[b]{0.24\textwidth}
            \centering
            \includegraphics[width=\textwidth]{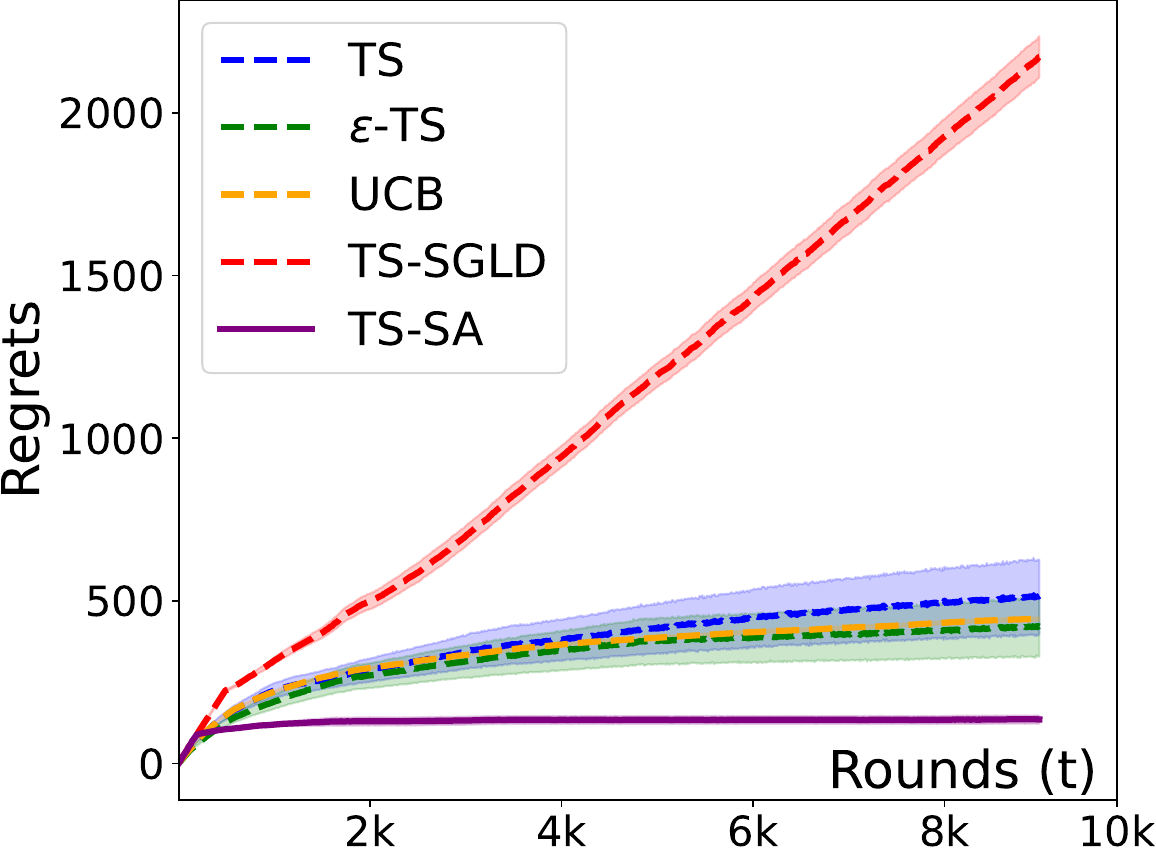}
        \caption{MGR,$\Delta = 0.5$,$K=10$}
        \end{subfigure}

        \begin{subfigure}[b]{0.24\textwidth}
            \centering
            \includegraphics[width=\textwidth]{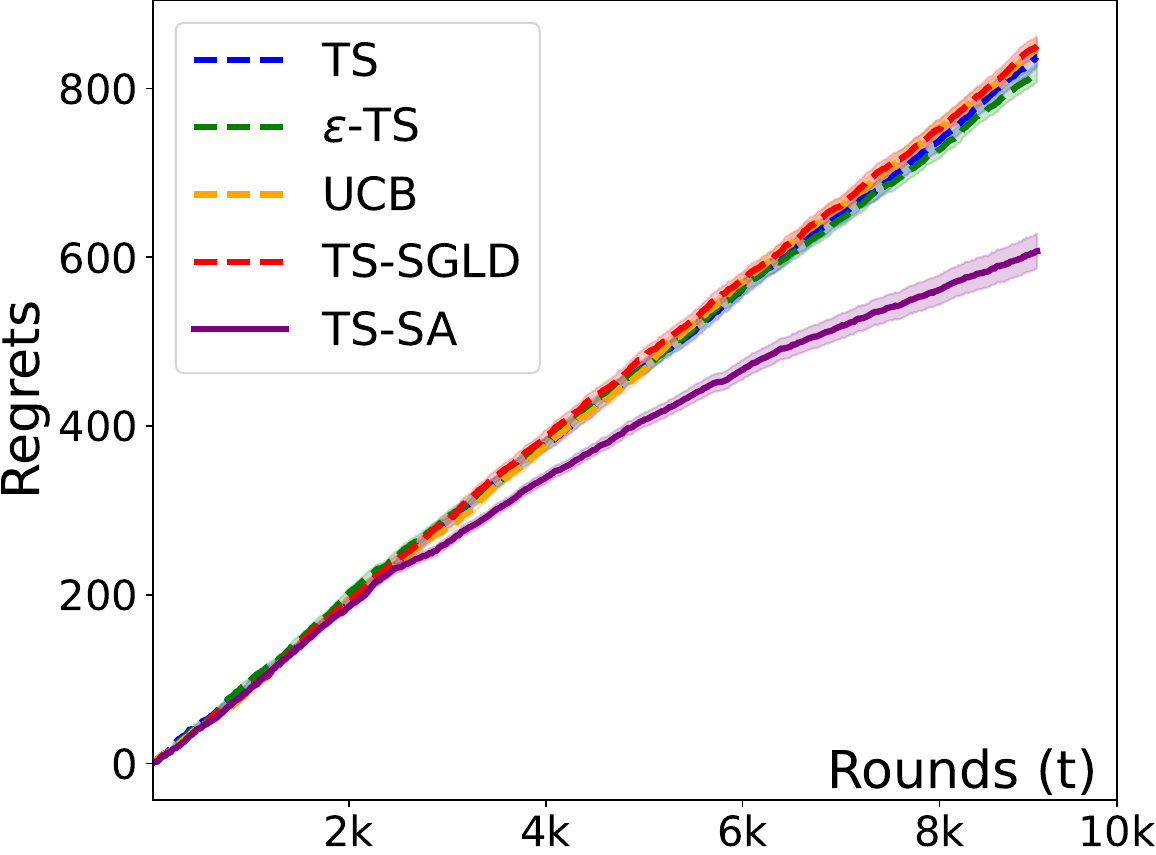}
        \caption{SGR,$\Delta = 0.1$,$K=50$}
        \end{subfigure}
        \begin{subfigure}[b]{0.24\textwidth}
            \centering
            \includegraphics[width=\textwidth]{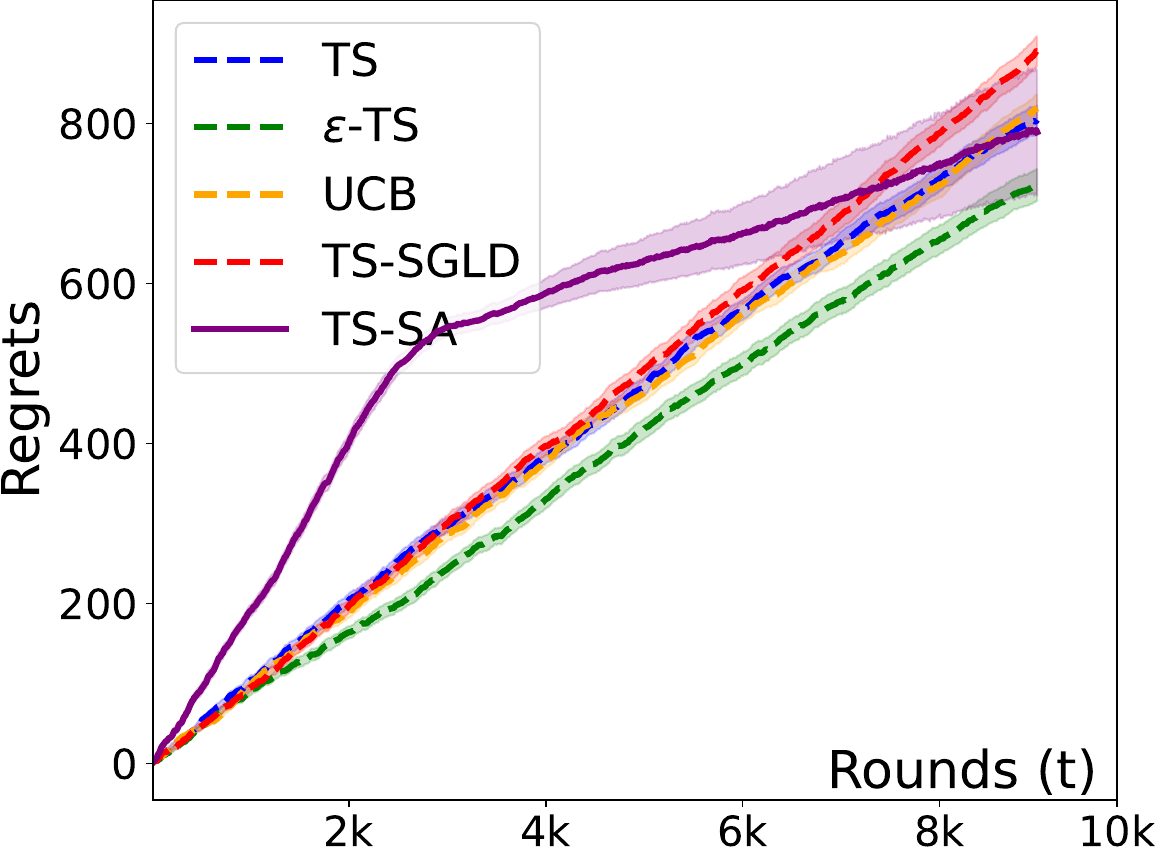}
        \caption{MGR,$\Delta = 0.1$,$K=50$}
        \end{subfigure}
        \begin{subfigure}[b]{0.24\textwidth}
            \centering
            \includegraphics[width=\textwidth]{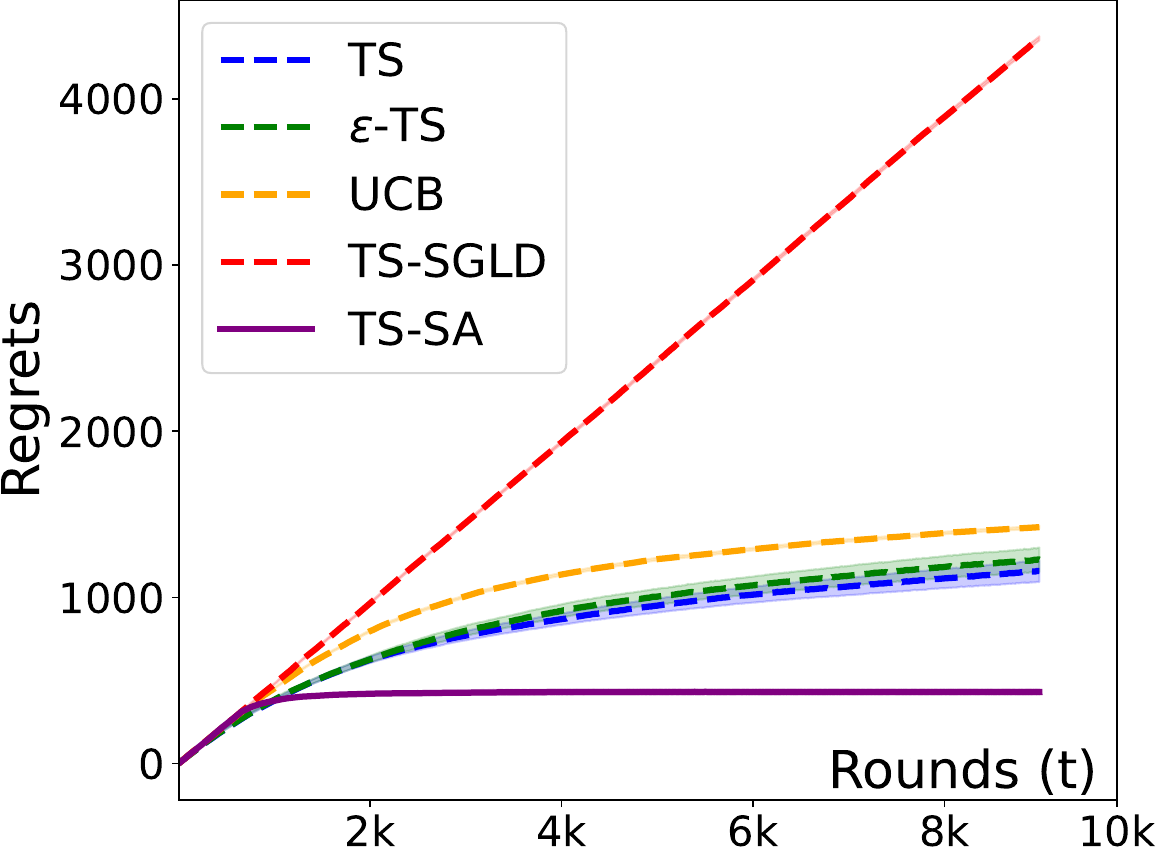}
        \caption{SGR,$\Delta = 0.5$,$K=50$}
        \end{subfigure}
        \begin{subfigure}[b]{0.24\textwidth}
            \centering
            \includegraphics[width=\textwidth]{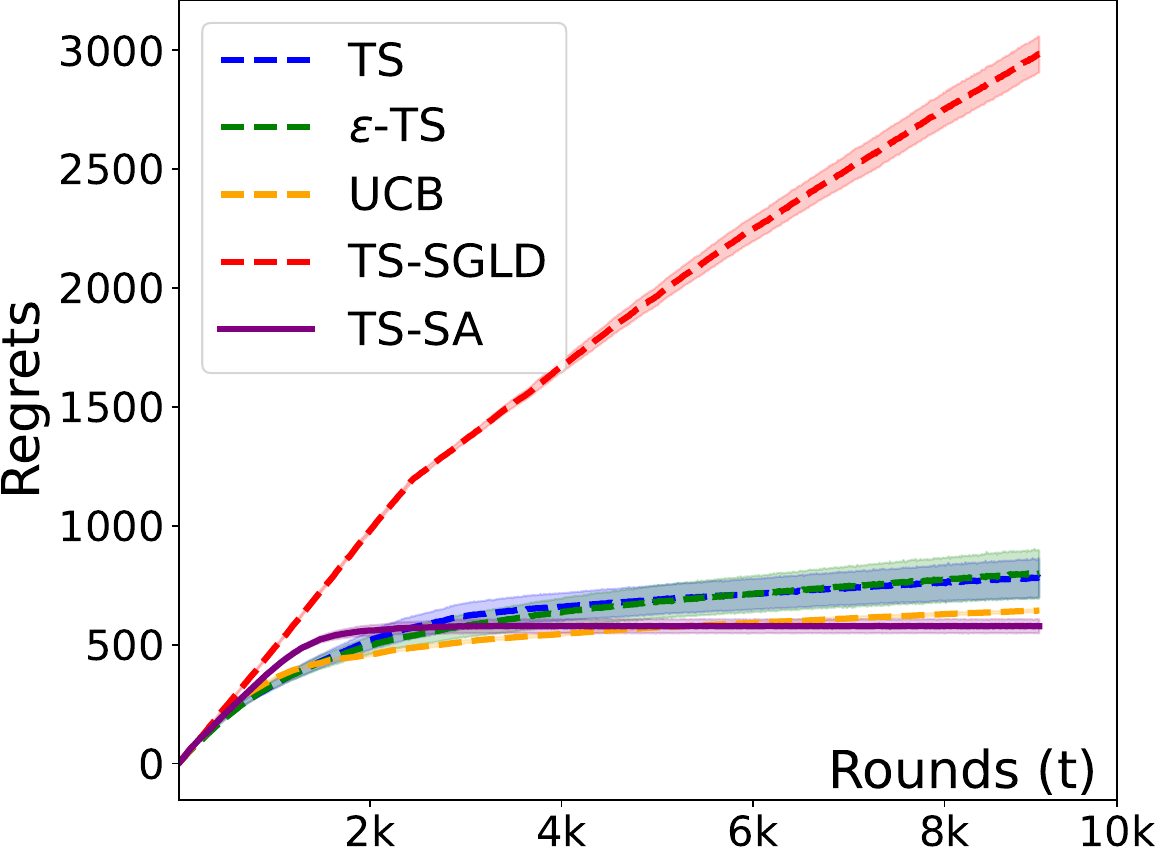}
        \caption{MGR,$\Delta = 0.5$,$K=50$}
        \end{subfigure}
    \vspace{-.00 in}
  \caption{Cumulative regret under different reward gaps $\Delta$, numbers of arms $K$, and reward settings.}\label{fig:results:main}
\vspace{-.10 in}
\end{figure*}

\paragraph{Regret Performance.}
\Cref{fig:results:main} compares cumulative regret across the eight experimental settings. In SGR with $\Delta=0.1$, TS-SA achieves lower final regret than the baselines for both $K=10$ and $K=50$. When $\Delta=0.5$, its improvement over the TS-SGLD baseline is particularly clear in both SGR and MGR. These results show that TS-SA can achieve competitive empirical performance using the single-update procedure. We next examine how additional updates and the choice of SA weights affect this performance.

\paragraph{Effect of SA Updates.}
The main implementation performs one LMC--SA update after receiving a new reward. To examine whether additional computation improves performance, we also repeat this update $N$ times using the same available rewards before making the next arm-selection decision. Each repetition starts from the estimate produced by the preceding update. \Cref{fig:sa_update_count} shows that increasing $N$ reduces regret, although the gains become smaller at larger $N$. These improvements require additional gradient computations, illustrating the trade-off between regret and computation per round.

The SA schedule determines how much each new Langevin proposal changes the current estimate. We vary the decay exponent $\alpha$ in the SA schedule $\gamma_n=\mathcal O(1/n^\alpha)$ while keeping the other parameters fixed: a larger $\alpha$ makes the SA weight decrease more quickly as the arm is observed more often. With $N=1$, \Cref{fig:sa_decay} shows lower mean regret for larger exponents over the tested range, with the lowest value obtained at $\alpha=1$. This choice has the same $1/n$ decay as the schedule used in our analysis, where $n$ is the number of observations from the arm. A broader sensitivity analysis is provided in \Cref{sec:sensitivity}, and additional ablation results are reported in \Cref{sec:ablation}.

\begin{figure}[!t]
    \centering
    \begin{subfigure}[b]{0.48\linewidth}
        \centering
        \includegraphics[width=\linewidth]{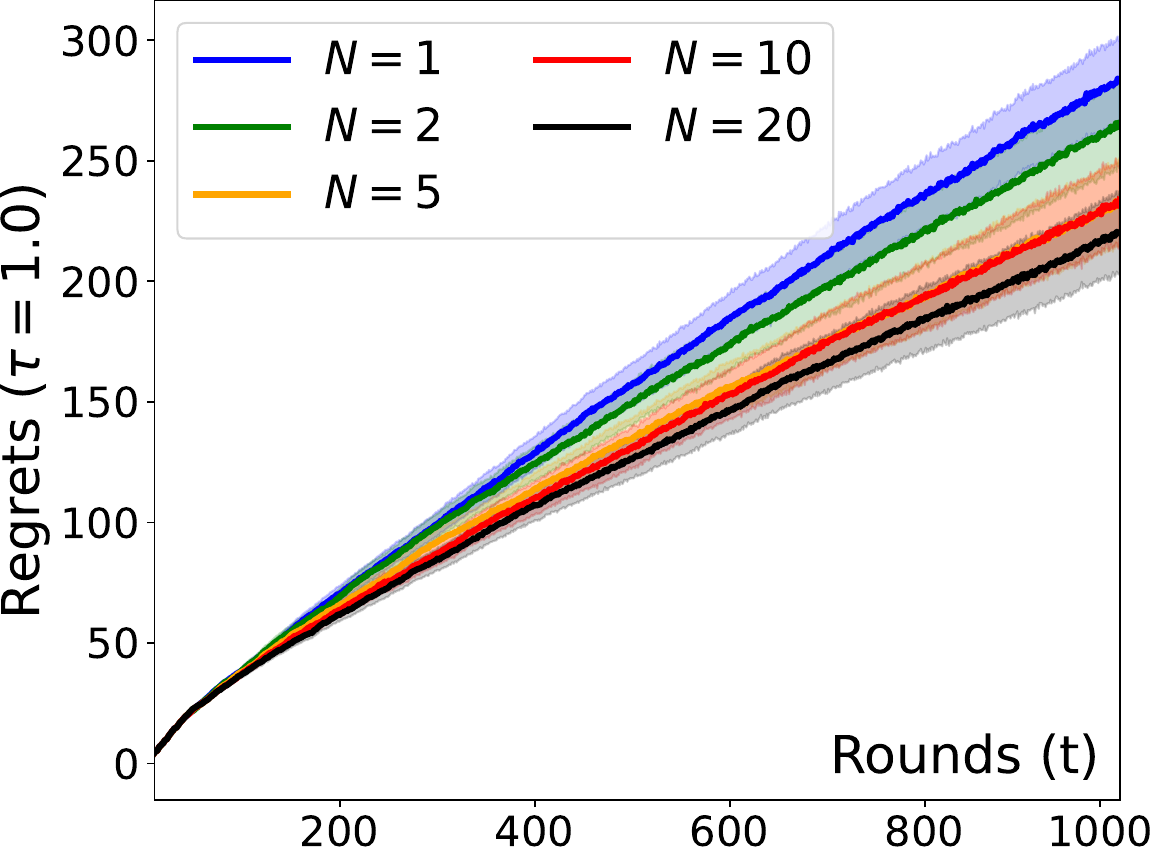}
        \caption{Updates per reward $N$.}
        \label{fig:sa_update_count}
    \end{subfigure}
    \hfill
    \begin{subfigure}[b]{0.48\linewidth}
        \centering
        \includegraphics[width=\linewidth]{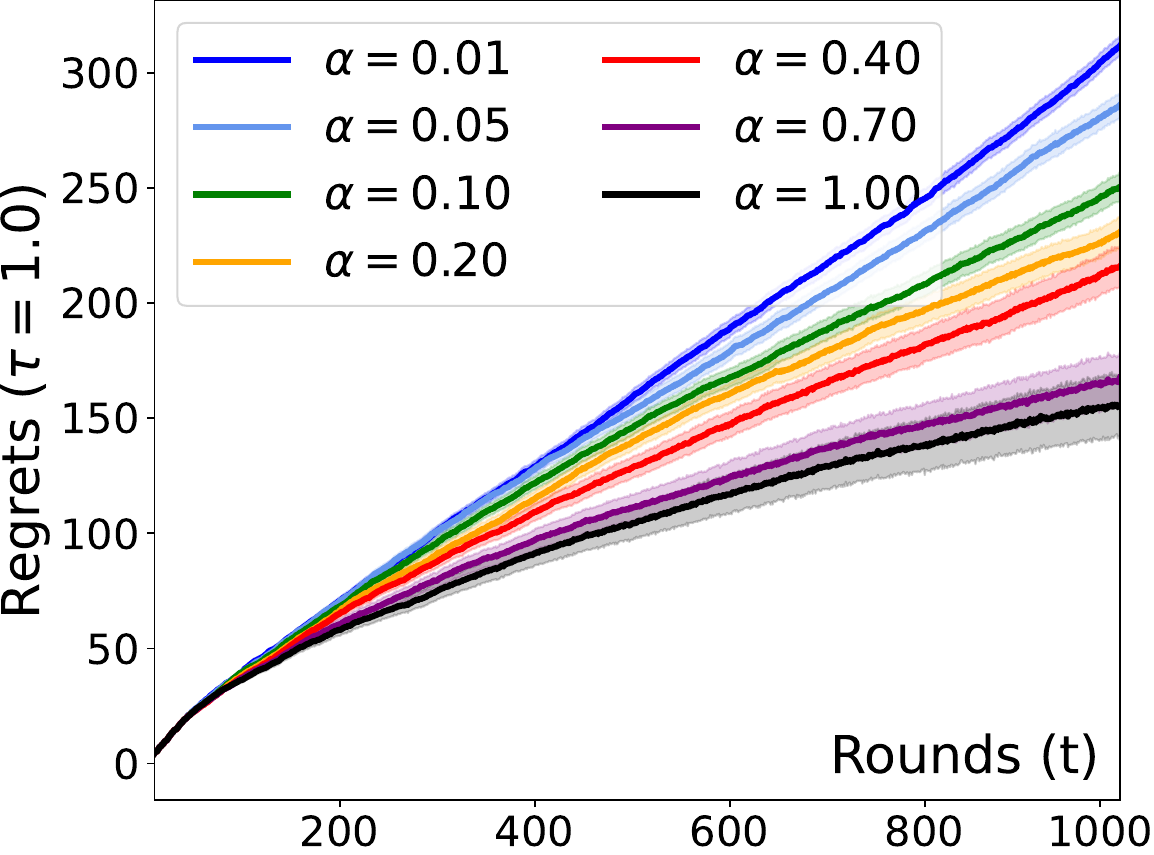}
        \caption{SA decay exponent $\alpha$.}
        \label{fig:sa_decay}
    \end{subfigure}
    \caption{Effects of repeated updates and the SA schedule in SGR, with $K=10$, $\Delta=0.5$, $\tau=1$, and $T=1,000$. Results are averaged over 100 independent trials. (a) We vary $N$ with $h=0.05$. (b) We vary $\alpha$ with $N=1$ and $\mathcal B=5$. Other parameters follow \Cref{sec:ablation}.}
    \label{fig:sa_update_ablation}
\end{figure}

\paragraph{Effect of Gradient Window.}

A gradient computed from a single reward can vary considerably across observations. In our experiments, we average the gradients of several recent rewards to reduce the influence of any single observation. The window size $\mathcal B$ specifies the maximum number of recent rewards retained for this average. The ablations in \Cref{sec:ablation} show that moderate window sizes give lower regret than using only one reward in the tested settings. This supports using several observations for each update while retaining only a fixed amount of reward data, so the storage requirement remains independent of the time horizon.

\paragraph{Failure Modes and Practical Guidelines.}

Our ablation studies reveal parameter settings that can lead to nearly linear regret growth over the horizon. First, small reward windows and insufficient warm-up observations can result in poor performance. Increasing $\mathcal B$ and $\Omega$ from very small values improves performance, with moderate choices around 20 giving good results in the tested settings. Second, overly large values of $\tau$ reduce the Gaussian perturbation used for arm selection and may lead to insufficient exploration.

\section{Conclusion and Outlook}
\label{sec:conclude}

This paper introduced TS-SA, a TS-based algorithm that combines Langevin proposals with stochastic approximation. Each new reward advances one continuing Markov chain for the selected arm through a single LMC--SA update. This formulation simplifies the convergence proof through stochastic approximation analysis. The resulting algorithm offers several advantages: (1) it maintains a \textit{memory cost independent of the time horizon}, since parameter updates require only the current estimate and the new reward; (2) it is supported by theoretical guarantees, achieving a \textit{near-optimal worst-case regret bound} of $\widetilde{\mathcal O}(\sqrt{KT})$ under the stated assumptions; and (3) it demonstrates \textit{strong empirical performance} compared with MAB baselines such as UCB, TS, and TS-SGLD.

While our method demonstrates strong empirical results, it still requires tuning several hyperparameters, including the proposal step size $h$ and the SA weight schedule. Future work may extend TS-SA to contextual bandits or reinforcement learning settings, where processing each new observation with a single update may support computationally efficient exploration.

\appendix

\section{Notations}
\label{sec:notations}

Before presenting the proof, we clarify the two indexing systems. At decision round $t$, $\btheta_{a,t}$ is the parameter sample used for arm selection, and $X_{t}$ is the reward of the selected arm. For each arm $a$, $X_a(n)$ is its $n$-th observed reward, and $\btheta_a(n)$ is its estimate after processing $n$ rewards. These arm-specific quantities change only when that arm is pulled.

We define the pull count as follows:
\begin{align}\label{equ:pull_count_definition}
    \mathcal T_a(t)=\sum_{u=1}^{t-1}\mathbb I(A_u=a),\qquad n_a=\mathcal T_a(t),\qquad\mathcal T_a(t+1)=\mathcal T_a(t)+\mathbb I(A_t=a).
\end{align}
Each arm is pulled once during initialization, so $\mathcal T_a(K+1)=1$. The total number of pulls of arm $a$ through horizon $T$ is $\mathcal T_a(T+1)$. Thus the pseudo-regret satisfies
\begin{align}\label{equ:regret_definition}
    \mathcal R(T)=\sum_{t=1}^T(\bar X_1-\bar X_{A_t}),\qquad\EE[\mathcal R(T)]=\sum_{a>1}\Delta_a\EE[\mathcal T_a(T+1)].
\end{align}

\paragraph{Initialization and Sources of Randomness.}
For each arm $a$, let $\{X_a(n)\}_{n\geq1}$ be i.i.d. rewards with density $p_a(\cdot|\btheta_a^*)$, and let $\boldsymbol{\xi}_{a,n}\sim\mathcal N(\zero,\Ib)$ denote the Gaussian noise used in its $n$-th update. With $h_n'=h\gamma_n$, the Langevin and SA steps give
\begin{align}\label{equ:joint_update}
    \btheta_a(n)=\btheta_a(n-1)+h_n'\nabla_{\btheta}\log p_a(X_a(n)|\btheta_a(n-1))+h_n'\sqrt{2/h}\,\boldsymbol{\xi}_{a,n},\qquad n\geq1.
\end{align}
The recursion starts from $\btheta_a(0)\sim\pi_a$ and defines $\btheta_a(n)$ for every $n\geq1$. At round $t$, the algorithm uses the estimate $\btheta_a(\mathcal T_a(t))$.

\begin{assumption}[Independence and initialization]
\label{assum:score_initialization}
The initial estimate satisfies $\btheta_a(0)\sim\pi_a=\mathcal N(\boldsymbol{\mu}_{0,a},\Sigma_{0,a})$, where $\boldsymbol{\mu}_{0,a}$ and $\Sigma_{0,a}$ are fixed independently of $T$. For each arm $a$, assume that a known constant $R_{0,a}>0$, independent of $T$, satisfies $\|\boldsymbol{\mu}_{0,a}-\btheta_a^*\|_2^2+\operatorname{Tr}(\Sigma_{0,a})\leq R_{0,a}^2$. All reward variables $X_a(n)$ and all standard Gaussian variables used to generate the update and arm-selection noises are mutually independent. The initial estimates are independent of these variables.
\end{assumption}

Set $v_{0,a}=\|\Sigma_{0,a}\|_{\mathrm{op}}=\lambda_{\max}(\Sigma_{0,a})$. Gaussian initialization then satisfies
\begin{align}\label{equ:initialization_bounds}
    \EE[\|\btheta_a(0)-\btheta_a^*\|_2^2]\leq R_{0,a}^2,\qquad \btheta_a(0)\sim\operatorname{GC}(v_{0,a}),
\end{align}
where $W\sim\operatorname{GC}(c)$ means that every real-valued $1$-Lipschitz function $f$ satisfies
\begin{align}\label{equ:gc_definition}
    \log\EE\exp\big(\lambda(f(W)-\EE[f(W)])\big)\leq\frac{\lambda^2c}{2},\qquad\lambda\in\mathbb R.
\end{align}

For the recursion of arm $a$, define $\mathcal H_{a,n}=\sigma(\btheta_a(0),X_a(1),\boldsymbol{\xi}_{a,1},\ldots,X_a(n),\boldsymbol{\xi}_{a,n})$. Let $\mathcal F_{t-1}$ denote the history before the round-$t$ decision, including the initial estimates, past actions, observed rewards, and all Gaussian noises used before round $t$. Both $\mathcal T_a(t)$ and $\btheta_a(\mathcal T_a(t))$ are $\mathcal F_{t-1}$-measurable, while the new Gaussian noises used for arm selection are independent of $\mathcal F_{t-1}$.

\paragraph{Parameter Choices.}
We choose the proposal step size $h$, the SA weights $\gamma_n$, and the exploration parameters $\tau_a$ as follows. First, fix $h>0$ and choose $\beta$ and $b$ satisfying
\begin{align}\label{equ:sa_stepsizes}
    h\beta>\max_{a\in\mathcal A}\frac1{m_a},\qquad b\geq\max_{a\in\mathcal A}\left\{1,\beta,\frac{2h\beta L_a^2}{m_a}\right\}.
\end{align}
We then set
\begin{align}\label{equ:sa_schedule}
    \gamma_n=\frac{\beta}{n+b},\qquad h_n'=h\gamma_n=\frac{h\beta}{n+b},\qquad n\geq1.
\end{align}
These choices ensure $0<\gamma_n\leq1$ and $h_n'\leq m_a/(2L_a^2)$, as required for the contraction bounds below. Finally, for each arm $a$, it suffices to choose
\begin{align}\label{equ:sa_temperature_explicit}
    0<\tau_a\leq\frac1{64}\min\left\{\frac1{bR_{0,a}^2},\frac{m_ah\beta-1}{dh\beta^2(hL_a^2/\nu_a+2)}\right\},\qquad a\in\mathcal A.
\end{align}
For the analysis, write $\rho_a=m_ah\beta>1$ and define
\begin{align}\label{equ:sa_constants}
    \sigma_a^2&=\frac{L_a^2}{\nu_a}+\frac2h,\qquad D_a=\frac{h^2\beta^2\sigma_a^2}{\rho_a-1},\qquad U_a=\max\{bR_{0,a}^2,dD_a\},\notag\\
    V_a&=\max\{bv_{0,a},D_a\},\qquad S_a^2=16(U_a+V_a),\qquad v_a=\frac{B_a^2}{\tau_a}.
\end{align}
The constants $U_a$ and $V_a$ control the mean-square error and fluctuations of the estimate, respectively. Since $v_{0,a}\leq\operatorname{Tr}(\Sigma_{0,a})\leq R_{0,a}^2$ and $d\geq1$, we have $V_a\leq U_a$ and hence $2S_a^2=32(U_a+V_a)\leq64U_a$. The explicit choice above is exactly $\tau_a\leq1/(64U_a)$ and therefore implies
\begin{align}\label{equ:sa_temperature}
    0<\tau_a\leq\frac1{2S_a^2},\qquad a\in\mathcal A.
\end{align}
This condition provides sufficient Gaussian exploration when the optimal arm's mean reward is underestimated, as shown in \Cref{lem:anti_concentration}.

\begin{table}[!htb]
\caption{Notation summary in the analysis of TS-SA.}\label{table_notation}
\centering
\small
\begin{tabular}{|p{0.24\linewidth}|p{0.67\linewidth}|}
\hline
Symbols & Explanations \\ \hline\hline
$h$, $\gamma_n$, $h_n'$ & Langevin, SA, and effective step sizes; $h_n'=h\gamma_n$. \\
$d$, $\pi_a$ & Parameter dimension and initialization distribution for arm $a$. \\
$\bphi_a$, $B_a$ & Arm feature and its norm $B_a=\|\bphi_a\|_2$. \\
$\bar X_a$, $\Delta_a$ & Mean reward and gap $\Delta_a=\bar X_1-\bar X_a$. \\
$m_a$, $\nu_a$, $L_a$ & Strong log-concavity constants in $\btheta$ and $X$, and the score Lipschitz constant. \\
$\tau_a$ & Controls the decision noise variance $(\tau_a n)^{-1}$ after $n$ pulls. \\
$\mathcal T_a(t)$ & Number of pulls of arm $a$ before round $t$. \\
$\btheta_a(n)$, $X_a(n)$ & Estimate after processing $n$ rewards and the $n$th reward of arm $a$. \\
$\btheta_{a,t}$, $\hat X_{a,t}$ & Parameter sample and arm-selection score $\hat X_{a,t}=\bphi_a^\top\btheta_{a,t}$. \\
$X_{t}$ & Reward observed in round $t$. \\
$\widehat\mu_a^{\text{SA}}(n)$ & Distribution of $\btheta_a(n)$ for fixed $n$. \\
$\widehat\mu_a^{\text{SA}}(n)[\tau_a]$ & Distribution after adding independent $\mathcal N(\zero,(\tau_a n)^{-1}\Ib)$ noise. \\
$\mathcal H_{a,n}$, $\mathcal F_{t-1}$ & History of arm $a$ through update $n$, and bandit history before round $t$. \\
$U_a$, $V_a$, $S_a$ & Constants in the mean-square and concentration bounds. \\
$v_a$ & Variance coefficient $B_a^2/\tau_a$ for the arm-selection score. \\
\hline
\end{tabular}
\end{table}

For fixed $n$, $\widehat\mu_a^{\text{SA}}(n)[\tau_a]$ includes randomness from both the estimate and the added Gaussian noise. Given $\mathcal F_{t-1}$, the pull count and current estimate are known, so only the new Gaussian noise remains random:
\begin{align*}
    \btheta_{a,t}\mid\mathcal F_{t-1}\sim\mathcal N\left(\btheta_a(\mathcal T_a(t)),\frac1{\tau_a\mathcal T_a(t)}\Ib\right).
\end{align*}
We use this conditional distribution in \Cref{sec:regret_analysis} to account for the dependence of pull counts on the observed history.

\section{Proof for Concentration and Convergence of TS-SA}
\label{appendix:sec:posterior:concentration}

In this section, we prove the concentration and convergence results for \Cref{alg:TS_SA}.

\paragraph{Proof Structure.}
We first establish a contraction bound for the update and a mean-square bound for the estimation error. We then use contraction to control how the initial estimate, rewards, and update noises affect the estimate. These bounds yield concentration around the true parameter and convergence in Wasserstein distance. Finally, we include the Gaussian noise used for arm selection. The exponential moment bound for the estimation error is also used in \Cref{lem:anti_concentration}.

\begin{lemma}[Mean-square convergence of TS-SA]
\label{lem:posterior_concentration_continuous_converged}
Suppose \Crefrange{assum:strongly_concave_theta}{assum:joint_lipschitz} and \Cref{assum:score_initialization} hold, and the parameters satisfy \eqref{equ:sa_stepsizes}. For every arm $a$ and fixed $n\geq1$, the update map
\begin{align*}
    \Psi_{a,n}(\btheta;X,\boldsymbol{\xi})=\btheta+h_n'\nabla_{\btheta}\log p_a(X|\btheta)+h_n'\sqrt{2/h}\,\boldsymbol{\xi}
\end{align*}
is $q_{a,n}$-Lipschitz in $\btheta$ for fixed $X$ and $\boldsymbol{\xi}$, where $q_{a,n}=\sqrt{1-\rho_a/(n+b)}\in(0,1)$. Moreover, for every arm $a$ and fixed $n\geq0$,
\begin{align}\label{equ:center_mean_square}
    \EE[\|\btheta_a(n)-\btheta_a^*\|_2^2]\leq\frac{U_a}{n+b}.
\end{align}
\end{lemma}

\begin{proof}
For notation simplicity, we fix the arm $a$ and suppress its index. Write $\btheta_n=\btheta_a(n)$, $X_n=X_a(n)$, $\mathcal H_n=\mathcal H_{a,n}$, and $\boldsymbol e_n=\btheta_n-\btheta^*$. Define that
\begin{align*}
    F(\btheta)=\EE_{X\sim p(\cdot|\btheta^*)}\left[\log\frac{p(X|\btheta)}{p(X|\btheta^*)}\right].
\end{align*}
By \Cref{assum:joint_lipschitz},
\begin{align*}
    \|\nabla_{\btheta}\log p(X|\btheta)\|_2\leq\|\nabla_{\btheta}\log p(0|\btheta^*)\|_2+L\big(|X|+\|\btheta-\btheta^*\|_2\big).
\end{align*}
The reward distribution has a finite first moment by \Cref{assum:strongly_concave_X}. The bound above is therefore integrable and uniform over bounded sets of parameters. Together with the mean value theorem, it shows that $F$ is finite and permits differentiation inside the expectation:
\begin{align*}
    \nabla_{\btheta}F(\btheta)=\EE_{X\sim p(\cdot|\btheta^*)}[\nabla_{\btheta}\log p(X|\btheta)].
\end{align*}
Moreover, Jensen's inequality gives
\begin{align*}
    F(\btheta)\leq\log\EE_{X\sim p(\cdot|\btheta^*)}\left[\frac{p(X|\btheta)}{p(X|\btheta^*)}\right]=\log\int_{\mathbb R}p(x|\btheta)\,dx=0=F(\btheta^*).
\end{align*}
Thus $F$ is maximized at $\btheta^*$, and $\nabla_{\btheta}F(\btheta^*)=\zero$. Taking expectations in \Cref{assum:strongly_concave_theta,assum:joint_lipschitz} also gives
\begin{align*}
    \langle\nabla_{\btheta}F(\btheta)-\nabla_{\btheta}F(\btheta'),\btheta-\btheta'\rangle&\leq-m\|\btheta-\btheta'\|_2^2,\\
    \|\nabla_{\btheta}F(\btheta)-\nabla_{\btheta}F(\btheta')\|_2&\leq\EE[\|\nabla_{\btheta}\log p(X|\btheta)-\nabla_{\btheta}\log p(X|\btheta')\|_2]\leq L\|\btheta-\btheta'\|_2.
\end{align*}

We first prove contraction of the update. Let $\boldsymbol\delta=\btheta-\btheta'$ and $\boldsymbol v=\nabla_{\btheta}\log p(X|\btheta)-\nabla_{\btheta}\log p(X|\btheta')$. When the two updates use the same $X$ and $\boldsymbol{\xi}$, the Gaussian terms cancel. Hence
\begin{align}\label{equ:update_contraction}
    \|\Psi_n(\btheta;X,\boldsymbol{\xi})-\Psi_n(\btheta';X,\boldsymbol{\xi})\|_2^2&=\|\boldsymbol\delta\|_2^2+2h_n'\langle\boldsymbol\delta,\boldsymbol v\rangle+(h_n')^2\|\boldsymbol v\|_2^2\notag\\
    &\leq\big(1-2mh_n'+L^2(h_n')^2\big)\|\boldsymbol\delta\|_2^2\notag\\
    &\leq(1-mh_n')\|\boldsymbol\delta\|_2^2=q_n^2\|\boldsymbol\delta\|_2^2.
\end{align}
The first inequality follows from strong log-concavity and Lipschitz continuity of the score. The second uses $L^2h_n'\leq m/2\leq m$. The same assumptions and Cauchy--Schwarz imply $m\leq L$, so $b\geq2h\beta L^2/m\geq2\rho$ and $0<q_n<1$.

The preceding bounds on $\nabla_{\btheta}F$ allow the same calculation for the update with the population score. Since $\nabla_{\btheta}F(\btheta^*)=\zero$, we obtain
\begin{align}\label{equ:population_contraction}
    \|\btheta+h_n'\nabla_{\btheta}F(\btheta)-\btheta^*\|_2^2\leq q_n^2\|\btheta-\btheta^*\|_2^2.
\end{align}

We next bound the noise introduced by one update. Rewrite \eqref{equ:joint_update} as
\begin{align}\label{equ:error_decomposition}
    \boldsymbol e_n&=\boldsymbol e_{n-1}+h_n'\nabla_{\btheta}F(\btheta_{n-1})+h_n'\boldsymbol M_n,\notag\\
    \boldsymbol M_n&=\nabla_{\btheta}\log p(X_n|\btheta_{n-1})-\nabla_{\btheta}F(\btheta_{n-1})+\sqrt{2/h}\,\boldsymbol{\xi}_n.
\end{align}
By \Cref{assum:score_initialization}, $X_n$ and $\boldsymbol{\xi}_n$ are independent of $\mathcal H_{n-1}$ and of each other. Since $\btheta_{n-1}$ is $\mathcal H_{n-1}$-measurable, $\EE[\boldsymbol M_n|\mathcal H_{n-1}]=\zero$.

Under \Cref{assum:strongly_concave_X}, the function $x\mapsto-\log p(x|\btheta^*)$ has second derivative with lower bound $\nu$. Therefore, \Cref{proposition:concentration} implies that the reward law satisfies $\operatorname{GC}(1/\nu)$. For fixed $\btheta$ and a unit vector $u$, let $f_{\btheta,u}(X)=u^\top\nabla_{\btheta}\log p(X|\btheta)$. By \Cref{assum:joint_lipschitz},
\begin{align*}
    |f_{\btheta,u}(X)-f_{\btheta,u}(X')|\leq\|\nabla_{\btheta}\log p(X|\btheta)-\nabla_{\btheta}\log p(X'|\btheta)\|_2\leq L|X-X'|.
\end{align*}
Set $\zeta_n=f_{\btheta_{n-1},u}(X_n)-u^\top\nabla_{\btheta}F(\btheta_{n-1})$. Given $\mathcal H_{n-1}$, the parameter $\btheta_{n-1}$ is fixed, so the concentration bound gives $\EE[e^{\lambda\zeta_n}|\mathcal H_{n-1}]\leq e^{\lambda^2L^2/(2\nu)}$. The independent Gaussian term $\sqrt{2/h}\,u^\top\boldsymbol{\xi}_n$ has variance $2/h$. Consequently, for every $\lambda\in\mathbb R$,
\begin{align}\label{equ:score_noise_mgf}
    \EE[e^{\lambda u^\top\boldsymbol M_n}|\mathcal H_{n-1}]=\EE[e^{\lambda\zeta_n}|\mathcal H_{n-1}]\,e^{\lambda^2/h}\leq\exp\left(\frac{\lambda^2L^2}{2\nu}+\frac{\lambda^2}{h}\right)=\exp\left(\frac{\lambda^2\sigma^2}{2}\right).
\end{align}
Taking $u$ to be each coordinate vector shows that every coordinate of $\boldsymbol M_n$ has conditional variance at most $\sigma^2$. Since its conditional mean is zero,
\begin{align}\label{equ:score_noise_second_moment}
    \EE[\|\boldsymbol M_n\|_2^2|\mathcal H_{n-1}]=\sum_{k=1}^d\EE[(\boldsymbol M_n)_k^2|\mathcal H_{n-1}]\leq d\sigma^2.
\end{align}

Given $\mathcal H_{n-1}$, the vector $\boldsymbol e_{n-1}+h_n'\nabla_{\btheta}F(\btheta_{n-1})$ is fixed. Its cross term with $\boldsymbol M_n$ therefore has conditional expectation zero. Expanding the square in \eqref{equ:error_decomposition} and applying \eqref{equ:population_contraction} and \eqref{equ:score_noise_second_moment} gives
\begin{align}\label{equ:mean_square_recursion}
    \EE[\|\boldsymbol e_n\|_2^2|\mathcal H_{n-1}]&=\|\boldsymbol e_{n-1}+h_n'\nabla_{\btheta}F(\btheta_{n-1})\|_2^2+(h_n')^2\EE[\|\boldsymbol M_n\|_2^2|\mathcal H_{n-1}]\notag\\
    &\leq\left(1-\frac{\rho}{n+b}\right)\|\boldsymbol e_{n-1}\|_2^2+\frac{d\sigma^2h^2\beta^2}{(n+b)^2}.
\end{align}

Set $u_n=\EE[\|\boldsymbol e_n\|_2^2]$. Taking expectations in \eqref{equ:mean_square_recursion} gives
\begin{align*}
    u_n\leq\left(1-\frac{\rho}{n+b}\right)u_{n-1}+\frac{d\sigma^2h^2\beta^2}{(n+b)^2}.
\end{align*}
We prove $u_n\leq U/(n+b)$ by induction. The case $n=0$ follows from $u_0\leq R_0^2\leq U/b$. Suppose $u_{n-1}\leq U/(n+b-1)$. Since $U\geq dh^2\beta^2\sigma^2/(\rho-1)$,
\begin{align*}
    \frac{U}{n+b}-\left(1-\frac{\rho}{n+b}\right)\frac{U}{n+b-1}=\frac{U(\rho-1)}{(n+b)(n+b-1)}\geq\frac{d\sigma^2h^2\beta^2}{(n+b)^2}.
\end{align*}
Substituting the induction hypothesis into the recursion therefore yields $u_n\leq U/(n+b)$. Restoring the arm index proves \eqref{equ:center_mean_square}.
\end{proof}

The preceding lemma bounds the mean-square estimation error. We next bound fluctuations around the mean for Lipschitz functions of the estimate by tracking the effect of each random input.

\begin{lemma}[Concentration of TS-SA estimates]
\label{lem:input_concentration}
Under the conditions of \Cref{lem:posterior_concentration_continuous_converged}, for every arm $a$, fixed $n\geq0$, $1$-Lipschitz function $f:\mathbb R^d\to\mathbb R$, and $\lambda\in\mathbb R$,
\begin{align}\label{equ:center_gc}
    \log\EE\exp\big(\lambda(f(\btheta_a(n))-\EE[f(\btheta_a(n))])\big)\leq\frac{\lambda^2V_a}{2(n+b)}.
\end{align}
\end{lemma}

\begin{proof}
Fix an arm and suppress its index. For fixed $n\geq0$, the estimate $\btheta_n$ is a function of the independent inputs $\btheta_0,X_1,\boldsymbol{\xi}_1,\ldots,X_n,\boldsymbol{\xi}_n$. Recall from \Cref{lem:posterior_concentration_continuous_converged} that $q_j=\sqrt{1-\rho/(j+b)}$ is the contraction factor for update $j$, we define that $P_{i,n}=\prod_{j=i+1}^n q_j$ for $0\leq i\leq n$, where an empty product equals $1$.

We first bound the effect of changing one reward. Fix $1\leq i\leq n$, replace $X_i$ by $X_i'$, and keep all other inputs unchanged. Write $\btheta_j'$ for the resulting estimates. The two estimates are same before update $i$, so \Cref{assum:joint_lipschitz} gives
\begin{align*}
    \|\btheta_i-\btheta_i'\|_2=h_i'\|\nabla_{\btheta}\log p(X_i|\btheta_{i-1})-\nabla_{\btheta}\log p(X_i'|\btheta_{i-1})\|_2\leq Lh_i'|X_i-X_i'|.
\end{align*}
For each subsequent update $j=i+1,\ldots,n$, the two estimates use the same reward and Gaussian noise. Applying \eqref{equ:update_contraction} to these updates gives
\begin{align*}
    \|\btheta_n-\btheta_n'\|_2\leq P_{i,n}\|\btheta_i-\btheta_i'\|_2\leq Lh_i'P_{i,n}|X_i-X_i'|.
\end{align*}
Changing only $\boldsymbol{\xi}_i$ produces a difference $h_i'\sqrt{2/h}(\boldsymbol{\xi}_i-\boldsymbol{\xi}_i')$ at update $i$, which is reduced by the same factor $P_{i,n}$ in the remaining updates. Changing only $\btheta_0$ gives a factor $P_{0,n}$ after all $n$ updates. Since $f$ is $1$-Lipschitz, the Lipschitz constants of $f(\btheta_n)$ with respect to its individual inputs are bounded by
\begin{align}\label{equ:input_sensitivities}
    c_0=P_{0,n},\qquad c_{X_i}=Lh_i'P_{i,n},\qquad c_{\boldsymbol{\xi}_i}=\sqrt{2/h}\,h_i'P_{i,n}.
\end{align}
These bounds hold for every choice of the other inputs.

By \eqref{equ:initialization_bounds} and \Cref{proposition:concentration}, we have $\btheta_0\sim\operatorname{GC}(v_0)$, $X_i\sim\operatorname{GC}(1/\nu)$, and $\boldsymbol{\xi}_i\sim\operatorname{GC}(1)$. Moreover, $|f(\btheta_n)|\leq|f(\btheta^*)|+\|\btheta_n-\btheta^*\|_2$, so $f(\btheta_n)$ is integrable by \eqref{equ:center_mean_square}. Applying \Cref{lemma:independent_input_concentration} therefore gives
\begin{align}\label{equ:input_mgf_bound}
    \log\EE[e^{\lambda(f(\btheta_n)-\EE[f(\btheta_n)])}]\leq\frac{\lambda^2\mathcal A_n}{2},
\end{align}
where
\begin{align*}
    \mathcal A_n =v_0c_0^2+\sum_{i=1}^n\left(\frac{c_{X_i}^2}{\nu}+c_{\boldsymbol{\xi}_i}^2\right)=v_0P_{0,n}^2+\sigma^2\sum_{i=1}^n(h_i')^2P_{i,n}^2.
\end{align*}
Here the second equality uses \eqref{equ:input_sensitivities} and $\sigma^2=L^2/\nu+2/h$.

It remains to bound $\mathcal A_n$. For $0\leq i<n$, we have $P_{i,n}=q_nP_{i,n-1}$, while $P_{n,n}=1$. Thus, for $n\geq1$,
\begin{align*}
    \mathcal A_n=q_n^2\mathcal A_{n-1}+\sigma^2(h_n')^2=\left(1-\frac{\rho}{n+b}\right)\mathcal A_{n-1}+\frac{\sigma^2h^2\beta^2}{(n+b)^2},\qquad \mathcal A_0=v_0.
\end{align*}
We prove $\mathcal A_n\leq V/(n+b)$ by induction. The initial case follows from $V\geq bv_0$. Suppose $\mathcal A_{n-1}\leq V/(n+b-1)$. Since $V\geq h^2\beta^2\sigma^2/(\rho-1)$,
\begin{align*}
    \frac{V}{n+b}-\left(1-\frac{\rho}{n+b}\right)\frac{V}{n+b-1}=\frac{V(\rho-1)}{(n+b)(n+b-1)}\geq\frac{\sigma^2h^2\beta^2}{(n+b)^2}.
\end{align*}
Substituting the induction hypothesis into the recursion gives $\mathcal A_n\leq V/(n+b)$. Combining this bound with \eqref{equ:input_mgf_bound} proves \eqref{equ:center_gc}.
\end{proof}

We combine the preceding two lemmas to bound the probability of a large estimation error. This gives convergence in Wasserstein distance and the exponential moment bound used in the regret analysis.

\begin{lemma}[Convergence of TS-SA]
\label{lem:TS-SA_convergence}
Under the conditions of \Cref{lem:posterior_concentration_continuous_converged}, for every arm $a$, fixed $n\geq1$, and $x\geq0$,
\begin{align}\label{equ:center_deviation}
    \PP\left(\|\btheta_a(n)-\btheta_a^*\|_2>\sqrt{\frac{U_a}{n+b}}+\sqrt{\frac{2V_ax}{n+b}}\right)\leq e^{-x}.
\end{align}
Moreover, for every $r\geq0$,
\begin{align}\label{equ:center_exponential_moment}
    \EE\exp\left(\frac{n\|\btheta_a(n)-\btheta_a^*\|_2^2}{S_a^2}\right)\leq2,\qquad \PP(\|\btheta_a(n)-\btheta_a^*\|_2>r)\leq2e^{-nr^2/S_a^2}.
\end{align}
Consequently, for every $p\geq1$,
\begin{align}\label{equ:convergence_of_TS-SA_bound_final}
    \mathcal W_p\big(\widehat\mu_a^{\text{SA}}(n),\delta_{\btheta_a^*}\big)=\big(\EE[\|\btheta_a(n)-\btheta_a^*\|_2^p]\big)^{1/p}\leq S_a\sqrt{\frac pn}.
\end{align}
\end{lemma}

\begin{proof}
Let $H_n=\|\btheta_a(n)-\btheta_a^*\|_2$. Since $\btheta\mapsto\|\btheta-\btheta_a^*\|_2$ is $1$-Lipschitz, \Cref{lem:input_concentration} and Markov's inequality give, for $r>0$ and $\lambda>0$,
\begin{align*}
    \PP(H_n-\EE[H_n]>r)\leq e^{-\lambda r}\EE[e^{\lambda(H_n-\EE[H_n])}]\leq\exp\left(-\lambda r+\frac{\lambda^2V_a}{2(n+b)}\right).
\end{align*}
Choosing $\lambda=r(n+b)/V_a$ yields $\PP(H_n-\EE[H_n]>r)\leq e^{-(n+b)r^2/(2V_a)}$. By Cauchy--Schwarz and \Cref{lem:posterior_concentration_continuous_converged},
\begin{align*}
    \EE[H_n]\leq\sqrt{\EE[H_n^2]}\leq\sqrt{\frac{U_a}{n+b}}.
\end{align*}
Combining these bounds and taking $r=\sqrt{2V_ax/(n+b)}$ proves \eqref{equ:center_deviation} for $x>0$. The case $x=0$ is immediate.

To obtain the exponential moment, set $W_n=\sqrt nH_n$ and $Z_n=(W_n-\sqrt{U_a})_+$, where $(z)_+=\max\{z,0\}$. Thus $Z_n$ measures how far the scaled error exceeds $\sqrt{U_a}$. For $z\geq0$, \eqref{equ:center_deviation} with $x=z^2/(2V_a)$ gives
\begin{align*}
    \PP(Z_n>z)=\PP\left(H_n>\frac{\sqrt{U_a}+z}{\sqrt n}\right) \leq\PP\left(H_n>\frac{\sqrt{U_a}+z}{\sqrt{n+b}}\right)\leq e^{-z^2/(2V_a)}.
\end{align*}
Since $W_n\leq\sqrt{U_a}+Z_n$, we have $W_n^2\leq2U_a+2Z_n^2$. The definition $S_a^2=16(U_a+V_a)$ gives $2U_a/S_a^2\leq1/8$ and $4V_a/S_a^2\leq1/4<1$. We therefore apply the first part of \Cref{lemma:exponential_moment_bounds} with $c=2/S_a^2$ to obtain
\begin{align*}
    \EE[e^{W_n^2/S_a^2}]\leq e^{2U_a/S_a^2}\EE[e^{2Z_n^2/S_a^2}]\leq\frac{e^{2U_a/S_a^2}}{1-4V_a/S_a^2}\leq\frac{e^{1/8}}{3/4}<2.
\end{align*}
Since $W_n^2=nH_n^2$, this proves the exponential moment bound in \eqref{equ:center_exponential_moment}. Markov's inequality then gives, for $r\geq0$,
\begin{align*}
    \PP(H_n>r)\leq e^{-nr^2/S_a^2}\EE[e^{nH_n^2/S_a^2}]\leq2e^{-nr^2/S_a^2}.
\end{align*}

The second part of \Cref{lemma:exponential_moment_bounds} gives
\begin{align*}
    \big(\EE[H_n^p]\big)^{1/p}=\frac{\big(\EE[W_n^p]\big)^{1/p}}{\sqrt n}\leq S_a\sqrt{\frac pn},\qquad p\geq1.
\end{align*}
Finally, in every coupling with $\delta_{\btheta_a^*}$, the second component equals $\btheta_a^*$ almost surely. Therefore,
\begin{align*}
    \mathcal W_p^p\big(\widehat\mu_a^{\text{SA}}(n),\delta_{\btheta_a^*}\big)=\EE[\|\btheta_a(n)-\btheta_a^*\|_2^p].
\end{align*}
Taking the $p$th root and applying the preceding moment bound proves \eqref{equ:convergence_of_TS-SA_bound_final}.
\end{proof}

We now combine the estimation error bound with a Gaussian norm bound to control the samples used for arm selection.

\begin{lemma}[Concentration of TS-SA decision samples]
\label{lem:posterior_concentration}
Under the conditions of \Cref{lem:TS-SA_convergence}, for every arm $a$, fixed $n\geq1$, and $\delta\in(0,1)$,
\begin{align}\label{equ:decision_concentration_sharp}
    \PP_{\btheta\sim\widehat\mu_a^{\text{SA}}(n)[\tau_a]}\left(\|\btheta-\btheta_a^*\|_2>\frac{S_a\sqrt{\log(4/\delta)}}{\sqrt n}+\frac{\sqrt d+\sqrt{2\log(2/\delta)}}{\sqrt{\tau_a n}}\right)\leq\delta.
\end{align}
In particular, with $C_1=\max_{a\in\mathcal A}(S_a+(\sqrt d+\sqrt2)/\sqrt{\tau_a})^2$,
\begin{align}\label{equ:posterior_concentration_subgaussian_tail}
    \PP_{\btheta\sim\widehat\mu_a^{\text{SA}}(n)[\tau_a]}\left(\|\btheta-\btheta_a^*\|_2>\sqrt{\frac{C_1}{n}\log\frac4\delta}\right)\leq\delta.
\end{align}
\end{lemma}

\begin{proof}
By the definition of $\widehat\mu_a^{\text{SA}}(n)[\tau_a]$, we may write $\btheta=\btheta_a(n)+Z/\sqrt{\tau_a n}$, where $Z\sim\mathcal N(\zero,\Ib)$ is independent of $\btheta_a(n)$. The triangle inequality gives
\begin{align*}
    \|\btheta-\btheta_a^*\|_2\leq\|\btheta_a(n)-\btheta_a^*\|_2+\frac{\|Z\|_2}{\sqrt{\tau_a n}}.
\end{align*}
By \eqref{equ:center_exponential_moment} and the Gaussian norm bound in \Cref{proposition:concentration},
\begin{align*}
    \PP\left(\|\btheta_a(n)-\btheta_a^*\|_2>\frac{S_a\sqrt{\log(4/\delta)}}{\sqrt n}\right)&\leq2e^{-\log(4/\delta)}=\frac\delta2,\\
    \PP\left(\|Z\|_2>\sqrt d+\sqrt{2\log(2/\delta)}\right)&\leq e^{-\log(2/\delta)}=\frac\delta2.
\end{align*}
By a union bound, both bounds hold with probability at least $1-\delta$. The triangle inequality then proves \eqref{equ:decision_concentration_sharp}.

Since $\log(4/\delta)\geq1$ and $\log(2/\delta)\leq\log(4/\delta)$,
\begin{align*}
    \frac{S_a\sqrt{\log(4/\delta)}}{\sqrt n}+\frac{\sqrt d+\sqrt{2\log(2/\delta)}}{\sqrt{\tau_a n}} \leq\left(S_a+\frac{\sqrt d+\sqrt2}{\sqrt{\tau_a}}\right)\sqrt{\frac{\log(4/\delta)}n} \leq\sqrt{\frac{C_1}{n}\log\frac4\delta}.
\end{align*}
This proves \eqref{equ:posterior_concentration_subgaussian_tail}.
\end{proof}

\section{Regret Analysis}
\label{sec:regret_analysis}

In this section, we prove the regret bound using the concentration results in \Cref{appendix:sec:posterior:concentration} and the Gaussian samples used for arm selection. We split the selections of each suboptimal arm according to whether its sampled reward exceeds a threshold. We then express both bounds as sums over pull counts and control them using concentration and anti-concentration.

For a fixed suboptimal arm $a$, let $\epsilon_a=\Delta_a/2$, $y_a=\bar X_1-\epsilon_a$, and $\mathcal E_a(t)=\{\hat X_{a,t}>y_a\}$, where $\hat X_{a,t}=\bphi_a^\top\btheta_{a,t}$. For any arm $j$, fixed $s\geq1$, and $\epsilon>0$, define
\begin{align}\label{equ:optimism_probability}
    \mathcal G_{j,s}(\epsilon)=\overline\Phi\left(\frac{\sqrt{\tau_js}}{B_j}\big(\bar X_1-\epsilon-\bphi_j^\top\btheta_j(s)\big)\right),
\end{align}
where $\Phi$ is the standard Gaussian cumulative distribution function and $\overline\Phi=1-\Phi$. The quantity $\mathcal G_{j,s}(\epsilon)$ depends on the random estimate $\btheta_j(s)$. By the Gaussian sampling rule,
\begin{align}\label{equ:conditional_optimism}
    &\hat X_{j,t}|\mathcal F_{t-1}\sim\mathcal N\left(\bphi_j^\top\btheta_j(\mathcal T_j(t)),\frac{B_j^2}{\tau_j\mathcal T_j(t)}\right),\notag\\
    &\PP(\hat X_{j,t}>\bar X_1-\epsilon|\mathcal F_{t-1})=\mathcal G_{j,\mathcal T_j(t)}(\epsilon).
\end{align}
Given $\mathcal F_{t-1}$, the estimates are fixed and only the new Gaussian samples are random. These samples are conditionally independent and have positive score variances because $B_j>0$ and $\tau_j>0$. 

\begin{lemma}[Regret decomposition]
\label{lem:regret_decompose}
For a fixed suboptimal arm $a$, define
\begin{align*}
    T_1=\EE\left[\sum_{t=K+1}^T\mathbb I(A_t=a,\mathcal E_a^c(t))\right],\qquad T_2=\EE\left[\sum_{t=K+1}^T\mathbb I(A_t=a,\mathcal E_a(t))\right].
\end{align*}
Then
\begin{align}\label{equ:regret_decomposition}
    \EE[\mathcal T_a(T+1)]=1+T_1+T_2,\qquad\EE[\mathcal R(T)]=\sum_{a>1}\Delta_a\EE[\mathcal T_a(T+1)].
\end{align}
\end{lemma}

\begin{proof}
Arm $a$ is pulled once during initialization. At every later round when $A_t=a$, exactly one of $\mathcal E_a(t)$ and $\mathcal E_a^c(t)$ holds. Therefore, on every sample path,
\begin{align*}
    \mathcal T_a(T+1)=1+\sum_{t=K+1}^T\mathbb I(A_t=a,\mathcal E_a^c(t))+\sum_{t=K+1}^T\mathbb I(A_t=a,\mathcal E_a(t)).
\end{align*}
Taking expectations proves the first identity. The second follows by grouping the terms of $\mathcal R(T)$ according to the selected arm, as in \eqref{equ:regret_definition}.
\end{proof}

The next two lemmas bound $T_1$ and $T_2$ by sums over pull counts.

\begin{lemma}[Bound on $T_1$]
\label{lem:regret-T1}
For every suboptimal arm $a$,
\begin{align*}
    T_1\leq\sum_{s=1}^T\EE\left[\frac1{\mathcal G_{1,s}(\epsilon_a)}-1\right].
\end{align*}
\end{lemma}

\begin{proof}
Recall that $A_t=\argmax_{j\in\mathcal A}\hat X_{j,t}$. Define $A_t'=\argmax_{j\neq1}\hat X_{j,t}$ and $Q_t=\{A_t'=a,\mathcal E_a^c(t)\}$. Also write
\begin{align*}
    p_t=\mathcal G_{1,\mathcal T_1(t)}(\epsilon_a),\qquad q_t=\PP(Q_t|\mathcal F_{t-1}).
\end{align*}
Given $\mathcal F_{t-1}$, the event $Q_t$ depends only on the suboptimal arms' Gaussian samples, which are independent of $\hat X_{1,t}$. If $Q_t$ occurs and $\hat X_{1,t}>y_a$, then arm $1$ has a higher sampled reward than every suboptimal arm, so $A_t=1$. Conversely, if $A_t=a$ and $\hat X_{a,t}\leq y_a$, then $Q_t$ occurs and $\hat X_{1,t}\leq y_a$. These two event inclusions give
\begin{align*}
    \PP(A_t=1|\mathcal F_{t-1})&\geq\PP(Q_t,\hat X_{1,t}>y_a|\mathcal F_{t-1})=q_tp_t,\\
    \PP(A_t=a,\mathcal E_a^c(t)|\mathcal F_{t-1})&\leq\PP(Q_t,\hat X_{1,t}\leq y_a|\mathcal F_{t-1})=q_t(1-p_t).
\end{align*}
Since the Gaussian distribution has a positive density on $\mathbb R$, $p_t>0$. Dividing the first inequality by $p_t$ and substituting into the second proves
\begin{align}\label{eq:bound_T1}
    \PP(A_t=a,\mathcal E_a^c(t)|\mathcal F_{t-1})\leq\left(\frac1{p_t}-1\right)\PP(A_t=1|\mathcal F_{t-1}).
\end{align}

The factor $1/p_t-1$ is nonnegative and $\mathcal F_{t-1}$-measurable. Taking expectations, summing over $t$, and applying the tower property gives
\begin{align}\label{eq:bound_T1_2}
    T_1&\leq\sum_{t=K+1}^T\EE\left[\left(\frac1{p_t}-1\right)\PP(A_t=1|\mathcal F_{t-1})\right]\notag\\
    &=\sum_{t=K+1}^T\EE\left[\EE\left[\left(\frac1{p_t}-1\right)\mathbb I(A_t=1)\,\middle|\,\mathcal F_{t-1}\right]\right]\notag\\
    &=\EE\left[\sum_{t=K+1}^T\left(\frac1{p_t}-1\right)\mathbb I(A_t=1)\right].
\end{align}
For any arm $j$ and fixed pull count $s$, let $N_{j,s}=\sum_{t=K+1}^T\mathbb I(A_t=j,\mathcal T_j(t)=s)$. Once arm $j$ is selected with count $s$, its count increases to $s+1$. Therefore, on every sample path,
\begin{align}\label{equ:pull_index_count}
    N_{j,s}\leq1.
\end{align}
Since $p_t=\mathcal G_{1,\mathcal T_1(t)}(\epsilon_a)$, grouping the rounds in which arm $1$ is selected by its count before the decision gives
\begin{align*}
    \sum_{t=K+1}^T\left(\frac1{p_t}-1\right)\mathbb I(A_t=1)=\sum_{s=1}^T\left(\frac1{\mathcal G_{1,s}(\epsilon_a)}-1\right)N_{1,s} \leq\sum_{s=1}^T\left(\frac1{\mathcal G_{1,s}(\epsilon_a)}-1\right).
\end{align*}
The recursion \eqref{equ:joint_update} defines $\btheta_1(s)$ for every fixed $s$, including pull counts not reached by time $T$. Taking expectations and using \eqref{eq:bound_T1_2} completes the proof.
\end{proof}

\begin{lemma}[Bound on $T_2$]
\label{lem:regret-T2}
For every suboptimal arm $a$,
\begin{align*}
    T_2\leq1+\sum_{s=1}^T\EE\left[\mathbb I\left(\mathcal G_{a,s}(\epsilon_a)>\frac1T\right)\right].
\end{align*}
\end{lemma}

\begin{proof}
Let $g_t=\mathcal G_{a,\mathcal T_a(t)}(\epsilon_a)$. Split the selected rounds according to whether $g_t>1/T$. Dropping the event $\mathcal E_a(t)$ from the first term and the event $A_t=a$ from the second term gives
\begin{align}\label{eq:bound_T2_2}
    T_2\leq\EE\left[\sum_{t=K+1}^T\mathbb I(A_t=a,g_t>1/T)\right]+\EE\left[\sum_{t=K+1}^T\mathbb I(\mathcal E_a(t),g_t\leq1/T)\right].
\end{align}
For the first term, the counting property \eqref{equ:pull_index_count} gives, on every sample path,
\begin{align*}
    \sum_{t=K+1}^T\mathbb I(A_t=a,g_t>1/T)&=\sum_{s=1}^T\mathbb I(\mathcal G_{a,s}(\epsilon_a)>1/T)N_{a,s}\\
    &\leq\sum_{s=1}^T\mathbb I(\mathcal G_{a,s}(\epsilon_a)>1/T).
\end{align*}
For the second term, $\{g_t\leq1/T\}$ is $\mathcal F_{t-1}$-measurable. By \eqref{equ:conditional_optimism}, $\PP(\mathcal E_a(t)|\mathcal F_{t-1})=g_t$. Hence
\begin{align*}
    \EE\left[\sum_{t=K+1}^T\mathbb I(\mathcal E_a(t),g_t\leq1/T)\right]&=\sum_{t=K+1}^T\EE[\mathbb I(g_t\leq1/T)\PP(\mathcal E_a(t)|\mathcal F_{t-1})]\\
    &=\sum_{t=K+1}^T\EE[g_t\mathbb I(g_t\leq1/T)]\leq\frac{T-K}{T}\leq1.
\end{align*}
Combining the two bounds with \eqref{eq:bound_T2_2} completes the proof.
\end{proof}

We next bound $\EE[1/\mathcal G_{1,s}(\epsilon)-1]$. The bound decreases exponentially with $s$, which allows us to sum it over the pull counts in \Cref{lem:regret-T1}.

\begin{lemma}[Anti-concentration of the optimal arm]
\label{lem:anti_concentration}
Under the conditions of \Cref{lem:TS-SA_convergence}, suppose $0<\tau_1\leq1/(2S_1^2)$. For every fixed $s\geq1$ and $\epsilon>0$,
\begin{align}\label{equ:inverse_optimism_bound}
    \EE\left[\frac1{\mathcal G_{1,s}(\epsilon)}-1\right]\leq C_0\exp\left(-\frac{\tau_1s\epsilon^2}{8B_1^2}\right),\qquad C_0=1+2e\sqrt{2\pi}.
\end{align}
In particular, $\EE[1/\mathcal G_{1,s}(\epsilon)]\leq1+C_0<30$.
\end{lemma}

\begin{proof}
We first derive two elementary Gaussian bounds. Let $\varphi$ be the standard Gaussian density and set $C_G=e\sqrt{2\pi}$. For $z\geq0$, changing variables from $z+u$ to $u$ gives
\begin{align*}
    \overline\Phi(z)=\int_0^\infty\varphi(z+u)\,du=e^{-z^2/2}\int_0^\infty e^{-zu}\varphi(u)\,du\leq\frac12e^{-z^2/2}.
\end{align*}
For a lower bound, $\varphi$ is decreasing on $[0,\infty)$, so
\begin{align*}
    \overline\Phi(z)\geq\int_z^{z+1}\varphi(u)\,du\geq\frac{e^{-(z+1)^2/2}}{\sqrt{2\pi}}\geq\frac{e^{-z^2}}{e\sqrt{2\pi}}.
\end{align*}
The last inequality follows from $(z+1)^2/2\leq z^2+1$, which is equivalent to $(z-1)^2\geq0$. If $z<0$, then $\overline\Phi(z)\geq1/2$. Let $z_+=\max\{z,0\}$ denote the positive part of $z$. Combining the bounds for $z\geq0$ and $z<0$ gives
\begin{align}\label{equ:gaussian_inverse_tail}
    \frac{1}{\overline\Phi(z)}\leq C_Ge^{z_+^2},\qquad z\in\mathbb R.
\end{align}

Define the normalized estimation error and margin by
\begin{align*}
    Y_s=\frac{\sqrt s}{B_1}\bphi_1^\top(\btheta_1^*-\btheta_1(s)),\qquad r_s=\frac{\epsilon\sqrt s}{B_1}.
\end{align*}
Cauchy--Schwarz implies $Y_s^2\leq s\|\btheta_1(s)-\btheta_1^*\|_2^2$. Thus \Cref{lem:TS-SA_convergence} and \eqref{equ:optimism_probability} give
\begin{align*}
    \EE[e^{Y_s^2/S_1^2}]\leq2,\qquad \mathcal G_{1,s}(\epsilon)=\overline\Phi\big(\sqrt{\tau_1}(Y_s-r_s)\big).
\end{align*}
We split the expectation according to whether $Y_s\leq r_s/2$.

On this event, $\sqrt{\tau_1}(Y_s-r_s)\leq-\sqrt{\tau_1}r_s/2$, and therefore $\mathcal G_{1,s}(\epsilon)\geq1/2$. By Gaussian symmetry and the upper tail bound above,
\begin{align}\label{equ:inverse_optimism_good}
    \frac1{\mathcal G_{1,s}(\epsilon)}-1 =\frac{1-\mathcal G_{1,s}(\epsilon)}{\mathcal G_{1,s}(\epsilon)}\leq2\big(1-\mathcal G_{1,s}(\epsilon)\big) \leq2\overline\Phi\left(\frac{\sqrt{\tau_1}r_s}{2}\right)\leq e^{-\tau_1r_s^2/8}.
\end{align}
On the complementary event, $Y_s>r_s/2>0$, so $(Y_s-r_s)_+\leq Y_s$. By \eqref{equ:gaussian_inverse_tail},
\begin{align*}
    \frac1{\mathcal G_{1,s}(\epsilon)}-1\leq\frac1{\mathcal G_{1,s}(\epsilon)}\leq C_Ge^{\tau_1(Y_s-r_s)_+^2}\leq C_Ge^{\tau_1Y_s^2}.
\end{align*}
Using Cauchy--Schwarz on the expectation restricted to this event gives
\begin{align}\label{equ:inverse_optimism_bad}
    \EE\left[\left(\frac1{\mathcal G_{1,s}(\epsilon)}-1\right)\mathbb I(Y_s>r_s/2)\right]&\leq C_G\big(\EE[e^{2\tau_1Y_s^2}]\big)^{1/2}\PP(Y_s>r_s/2)^{1/2}\notag\\
    &\leq2C_Ge^{-r_s^2/(8S_1^2)}\leq2C_Ge^{-\tau_1r_s^2/8}.
\end{align}
To justify the second inequality, $2\tau_1\leq1/S_1^2$ gives $\EE[e^{2\tau_1Y_s^2}]\leq2$, while Markov's inequality gives
\begin{align*}
    \PP(Y_s>r_s/2)\leq\PP(Y_s^2>r_s^2/4)\leq e^{-r_s^2/(4S_1^2)}\EE[e^{Y_s^2/S_1^2}]\leq2e^{-r_s^2/(4S_1^2)}.
\end{align*}
The last inequality in \eqref{equ:inverse_optimism_bad} uses $\tau_1\leq1/S_1^2$. Finally, taking expectations of \eqref{equ:inverse_optimism_good} on its event and adding \eqref{equ:inverse_optimism_bad} yields
\begin{align*}
    \EE\left[\frac1{\mathcal G_{1,s}(\epsilon)}-1\right]\leq(1+2C_G)e^{-\tau_1r_s^2/8}=C_0e^{-\tau_1s\epsilon^2/(8B_1^2)}.
\end{align*}
This proves \eqref{equ:inverse_optimism_bound} and also implies the stated uniform bound.
\end{proof}

The exponential moment in \eqref{equ:inverse_optimism_bad} controls the contribution from large estimation errors. A mean-square error bound alone is not sufficient for this step.

\begin{lemma}[Bounds on the two sums]
\label{lem:bound_two_terms}
Suppose the conditions of \Cref{lem:TS-SA_convergence} hold and the parameters satisfy \eqref{equ:sa_temperature}. For each $a>1$, let $\ell_a=\lceil32v_a\log T/\Delta_a^2\rceil$. Then
\begin{align}\label{equ:two_sum_bounds}
    \sum_{s=1}^T\EE\left[\frac1{\mathcal G_{1,s}(\epsilon_a)}-1\right]&\leq\frac{32C_0v_1}{\Delta_a^2},\notag\\
    \sum_{s=1}^T\EE[\mathbb I(\mathcal G_{a,s}(\epsilon_a)>1/T)]&\leq\ell_a+\frac{32B_a^2S_a^2}{\Delta_a^2}.
\end{align}
\end{lemma}

\begin{proof}
For the first sum, apply \Cref{lem:anti_concentration} with $\epsilon_a=\Delta_a/2$. Using $v_1=B_1^2/\tau_1$, we obtain
\begin{align*}
    \sum_{s=1}^T\EE\left[\frac1{\mathcal G_{1,s}(\epsilon_a)}-1\right]\leq C_0\sum_{s=1}^\infty e^{-s\Delta_a^2/(32v_1)}=\frac{C_0}{e^{\Delta_a^2/(32v_1)}-1}\leq\frac{32C_0v_1}{\Delta_a^2}.
\end{align*}
The last inequality follows from $e^c-1\geq c$ for $c>0$.

For the second sum, define $H_{a,s}=\{\bphi_a^\top\btheta_a(s)\leq\bar X_a+\Delta_a/4\}$. On this event, the threshold $y_a$ is at least $\Delta_a/4$ above the estimated reward, because
\begin{align*}
    \bar X_1-\epsilon_a-\bphi_a^\top\btheta_a(s)\geq\bar X_a+\Delta_a/2-(\bar X_a+\Delta_a/4)=\Delta_a/4.
\end{align*}
By \eqref{equ:optimism_probability}, on $H_{a,s}$ we have
\begin{align}\label{equ:good_center_optimism}
    \mathcal G_{a,s}(\epsilon_a)\leq\overline\Phi\left(\frac{\Delta_a\sqrt{\tau_as}}{4B_a}\right)\leq\exp\left(-\frac{s\Delta_a^2}{32v_a}\right).
\end{align}
For $s\geq\ell_a$, the last expression is at most $e^{-\log T}=1/T$. Consequently,
\begin{align*}
    \mathbb I(\mathcal G_{a,s}(\epsilon_a)>1/T)\leq\mathbb I(s<\ell_a)+\mathbb I(H_{a,s}^c).
\end{align*}
On $H_{a,s}^c$, $\bphi_a^\top(\btheta_a(s)-\btheta_a^*)>\Delta_a/4$. Since this projection is at most $B_a\|\btheta_a(s)-\btheta_a^*\|_2$, \eqref{equ:center_exponential_moment} implies
\begin{align*}
    \PP(H_{a,s}^c)\leq\PP\left(\|\btheta_a(s)-\btheta_a^*\|_2>\frac{\Delta_a}{4B_a}\right)\leq2e^{-s\Delta_a^2/(16B_a^2S_a^2)}.
\end{align*}
Summing over the fixed pull counts and using the same geometric-series bound gives
\begin{align*}
    \sum_{s=1}^T\EE[\mathbb I(\mathcal G_{a,s}(\epsilon_a)>1/T)]\leq\ell_a+\sum_{s=1}^\infty\PP(H_{a,s}^c) \leq\ell_a+2\sum_{s=1}^\infty e^{-s\Delta_a^2/(16B_a^2S_a^2)}\leq\ell_a+\frac{32B_a^2S_a^2}{\Delta_a^2}.
\end{align*}
This proves the second bound. The argument uses \eqref{equ:good_center_optimism} for each estimate in $H_{a,s}$ and bounds $\PP(H_{a,s}^c)$ using concentration at the fixed pull count $s$.
\end{proof}

Finally, we combine the preceding results to obtain the regret bound.

\begin{theorem}[Regret bound]
\label{lem:regret_bound}
Under the conditions of \Cref{lem:bound_two_terms}, for every $T\geq K$ and $a>1$,
\begin{align}\label{equ:pull_count_bound}
    \EE[\mathcal T_a(T+1)]\leq3+\frac{32}{\Delta_a^2}\big(v_a\log(eT)+C_0v_1\big).
\end{align}
Consequently,
\begin{align}\label{equ:problem_dependent_bound}
    \EE[\mathcal R(T)]\leq\sum_{a>1}\left[3\Delta_a+\frac{32}{\Delta_a}\big(v_a\log(eT)+C_0v_1\big)\right].
\end{align}
With $v_\star=\max_a v_a$, $\Delta_{\max}=\max_{a>1}\Delta_a$, and $C_\star=32(1+C_0)$, we also have
\begin{align}\label{equ:problem_independent_bound}
    \EE[\mathcal R(T)]\leq2\sqrt{C_\star v_\star KT\log(eT)}+3K\Delta_{\max}.
\end{align}
\end{theorem}

\begin{proof}
By \Cref{lem:regret_decompose,lem:regret-T1,lem:regret-T2,lem:bound_two_terms},
\begin{align*}
    \EE[\mathcal T_a(T+1)]&=1+T_1+T_2\\
    &\leq2+\ell_a+\frac{32C_0v_1+32B_a^2S_a^2}{\Delta_a^2}\\
    &\leq3+\frac{32v_a\log T+16v_a+32C_0v_1}{\Delta_a^2}\\
    &\leq3+\frac{32}{\Delta_a^2}\big(v_a\log(eT)+C_0v_1\big).
\end{align*}
The second inequality uses $\ell_a\leq1+32v_a\log T/\Delta_a^2$ and $B_a^2S_a^2\leq B_a^2/(2\tau_a)=v_a/2$, which follows from \eqref{equ:sa_temperature}. The last inequality uses $\log(eT)=1+\log T$. This proves \eqref{equ:pull_count_bound}. Multiplying it by $\Delta_a$ and summing over $a>1$ proves \eqref{equ:problem_dependent_bound}.

For the instance-independent bound, fix $\delta>0$ and split the suboptimal arms into $\mathcal A_{\leq\delta}=\{a>1:\Delta_a\leq\delta\}$ and $\mathcal A_{>\delta}=\{a>1:\Delta_a>\delta\}$. Since there are exactly $T$ pulls in total,
\begin{align}\label{equ:problem_independent_bound_part_1}
    \sum_{a\in\mathcal A_{\leq\delta}}\Delta_a\EE[\mathcal T_a(T+1)]\leq\delta\EE\left[\sum_{a\in\mathcal A_{\leq\delta}}\mathcal T_a(T+1)\right]\leq T\delta.
\end{align}
For the remaining arms, $\Delta_a>\delta$ and $\log(eT)\geq1$ imply
\begin{align}\label{equ:problem_independent_bound_part_2}
    \sum_{a\in\mathcal A_{>\delta}}\Delta_a\EE[\mathcal T_a(T+1)]&\leq\sum_{a\in\mathcal A_{>\delta}}\left[3\Delta_a+\frac{32}{\Delta_a}\big(v_a\log(eT)+C_0v_1\big)\right]\notag\\
    &\leq3K\Delta_{\max}+\frac{32(1+C_0)v_\star K\log(eT)}{\delta}.
\end{align}
Adding \eqref{equ:problem_independent_bound_part_1} and \eqref{equ:problem_independent_bound_part_2}, and recalling $C_\star=32(1+C_0)$, gives
\begin{align*}
    \EE[\mathcal R(T)]\leq T\delta+\frac{C_\star v_\star K\log(eT)}{\delta}+3K\Delta_{\max}.
\end{align*}
Choosing $\delta=\sqrt{C_\star v_\star K\log(eT)/T}$ makes the first two terms equal and proves \eqref{equ:problem_independent_bound}. Furthermore, \Cref{assum:score_initialization} gives $|\bar X_a|\leq|\bphi_a^\top\boldsymbol{\mu}_{0,a}|+B_aR_{0,a}$, so
\begin{align*}
    \Delta_{\max}\leq C_3:=2\max_{a\in\mathcal A}\left\{|\bphi_a^\top\boldsymbol{\mu}_{0,a}|+B_aR_{0,a}\right\}.
\end{align*}
Here $C_3$ depends only on the known features and initialization parameters. Since $T\geq K$, we obtain
\begin{align*}
    \EE[\mathcal R(T)]\leq2\sqrt{C_\star v_\star KT\log(eT)}+3C_3K=\widetilde{\mathcal O}(\sqrt{KT}),
\end{align*}
where the implicit constant depends on $v_\star$ and $C_3$. This establishes the worst-case regret bound and completes the proof.
\end{proof}

\section{Auxiliary Lemmas}
\label{appendix:sec:auxiliary}

We provide the concentration and moment estimates used in \Cref{appendix:sec:posterior:concentration}. Gaussian concentration is expressed through the moment-generating inequality in \eqref{equ:gc_definition}.

\begin{lemma}[Gaussian concentration under strong log-concavity]
\label{proposition:concentration}
Let $W\in\mathbb R^q$ have a positive density $p_W$ with $\log p_W\in C^2(\mathbb R^q)$ and $\nabla^2\log p_W(w)\preceq-\nu\Ib$ for every $w\in\mathbb R^q$, where $\nu>0$. For every $L$-Lipschitz function $f:\mathbb R^q\to\mathbb R$,
\begin{align}\label{equ:strong_logconcave_mgf}
    \log\EE \Big[e^{\lambda(f(W)-\EE[f(W)])}\Big]\leq\frac{\lambda^2L^2}{2\nu},\qquad\lambda\in\mathbb R.
\end{align}
In particular, $\operatorname{Var}(f(W))\leq L^2/\nu$ and, for $L>0$ and $r\geq0$, $\PP(f(W)-\EE[f(W)]>r)\leq e^{-\nu r^2/(2L^2)}$. If $Z\sim\mathcal N(\zero,\Ib)$ in $\mathbb R^q$, then
\begin{align}\label{equ:gaussian_norm_tail}
    \PP(\|Z\|_2>\sqrt q+\sqrt{2x})\leq e^{-x},\qquad x\geq0.
\end{align}
\end{lemma}

\begin{proof}
By the log-Sobolev inequality for strongly log-concave densities, for smooth $u$,
\begin{align*}
    \operatorname{Ent}(e^{u(W)})\leq\frac1{2\nu}\EE\Big[e^{u(W)}\|\nabla u(W)\|_2^2\Big],\qquad\operatorname{Ent}(Y)=\EE[Y\log Y]-\EE[Y]\log\EE[Y].
\end{align*}
First take $f$ bounded and smooth with $\|\nabla f\|_2\leq L$. Let $Y=f(W)-\EE[f(W)]$ and $\psi(\lambda)=\log\EE[e^{\lambda Y}]$. Applying the inequality with $u=\lambda(f-\EE[f(W)])$ gives
\begin{align*}
    \operatorname{Ent}(e^{\lambda Y})=\lambda\EE[Ye^{\lambda Y}]-\EE[e^{\lambda Y}]\log\EE[e^{\lambda Y}]\leq\frac{\lambda^2L^2}{2\nu}\EE[e^{\lambda Y}].
\end{align*}
Divide by $\EE[e^{\lambda Y}]>0$. Since $\psi'(\lambda)=\EE[Ye^{\lambda Y}]/\EE[e^{\lambda Y}]$, this yields $\lambda\psi'(\lambda)-\psi(\lambda)\leq\lambda^2L^2/(2\nu)$. For $\lambda>0$,
\begin{align*}
    \frac{d}{d\lambda}\left(\frac{\psi(\lambda)}{\lambda}\right)=\frac{\lambda\psi'(\lambda)-\psi(\lambda)}{\lambda^2}\leq\frac{L^2}{2\nu}.
\end{align*}
Because $\psi(0)=\psi'(0)=0$, the ratio $\psi(\lambda)/\lambda$ tends to $0$ as $\lambda\downarrow0$. Integrating the last inequality from $0$ to $\lambda$ gives $\psi(\lambda)\leq\lambda^2L^2/(2\nu)$. Applying the same argument to $-f$ proves the claim for $\lambda<0$, and $\lambda=0$ is immediate. Strong log-concavity gives $\EE[\|W\|_2]<\infty$, so every Lipschitz $f$ is integrable. Truncation and smooth approximation preserving the Lipschitz constant extend the result to such $f$; convergence of the means and Fatou's lemma preserve the moment-generating bound.

Expanding $\psi$ at zero gives $\psi(\lambda)=\lambda^2\operatorname{Var}(f(W))/2+o(\lambda^2)$, proving the variance bound. For $L>0$, Markov's inequality and \eqref{equ:strong_logconcave_mgf} give
\begin{align*}
    \PP(Y>r)\leq\inf_{\lambda>0}\exp\left(-\lambda r+\frac{\lambda^2L^2}{2\nu}\right)=e^{-\nu r^2/(2L^2)},
\end{align*}
where the minimizer is $\lambda=\nu r/L^2$ when $r>0$, and the case $r=0$ follows by a limit. For $L=0$, the function is constant and its deviation probability is zero.

For $Z\sim\mathcal N(\zero,\Ib)$, we have $\nabla^2\log p_Z(z)=-\Ib$, and $f(z)=\|z\|_2$ is $1$-Lipschitz. Furthermore, $\EE[\|Z\|_2]\leq\sqrt{\EE[\|Z\|_2^2]}=\sqrt q$. Taking $r=\sqrt{2x}$ in the one-sided tail bound proves \eqref{equ:gaussian_norm_tail}.
\end{proof}

\begin{lemma}[Concentration through independent inputs]
\label{lemma:independent_input_concentration}
Let $W_0,\ldots,W_m$ be independent random vectors with $W_i\sim\operatorname{GC}(\kappa_i)$. Suppose $H(W_0,\ldots,W_m)$ is integrable and is $c_i$-Lipschitz in its $i$th argument, uniformly over the other arguments. Then
\begin{align}\label{equ:independent_input_mgf}
    \log\EE[e^{\lambda(H-\EE[H])}]\leq\frac{\lambda^2}{2}\sum_{i=0}^m\kappa_ic_i^2,\qquad\lambda\in\mathbb R.
\end{align}
\end{lemma}

\begin{proof}
First suppose $H$ is bounded. Let $\mathcal D_i=\sigma(W_0,\ldots,W_i)$, with $\mathcal D_{-1}$ trivial. The Doob martingale decomposition is
\begin{align*}
    H-\EE[H]=\sum_{i=0}^mD_i,\qquad D_i=\EE[H|\mathcal D_i]-\EE[H|\mathcal D_{i-1}].
\end{align*}
For fixed $W_0,\ldots,W_{i-1}$, define
\begin{align*}
    f_i(w)=\EE[H(W_0,\ldots,W_{i-1},w,W_{i+1},\ldots,W_m)|\mathcal D_{i-1}],
\end{align*}
where the expectation is over $W_{i+1},\ldots,W_m$. For any $w,w'$,
\begin{align*}
    |f_i(w)-f_i(w')|\leq\EE\big[|H(W_0,\ldots,w,\ldots,W_m)-H(W_0,\ldots,w',\ldots,W_m)|\,\big|\,\mathcal D_{i-1}\big]\leq c_i\|w-w'\|_2.
\end{align*}
Thus $f_i$ is $c_i$-Lipschitz for every fixed choice of the earlier inputs. Independence gives $D_i=f_i(W_i)-\EE[f_i(W_i)|\mathcal D_{i-1}]$. Applying the concentration property of $W_i$ conditionally yields
\begin{align*}
    \EE[e^{\lambda D_i}|\mathcal D_{i-1}]\leq e^{\lambda^2\kappa_ic_i^2/2}.
\end{align*}
The partial sum $\sum_{i=0}^{m-1}D_i$ is $\mathcal D_{m-1}$-measurable. Therefore,
\begin{align*}
    \EE\left[e^{\lambda\sum_{i=0}^mD_i}\right]&=\EE\left[e^{\lambda\sum_{i=0}^{m-1}D_i}\EE[e^{\lambda D_m}|\mathcal D_{m-1}]\right]\\
    &\leq e^{\lambda^2\kappa_mc_m^2/2}\EE\left[e^{\lambda\sum_{i=0}^{m-1}D_i}\right].
\end{align*}
Repeating the same step for $m-1,\ldots,0$ gives $\EE[e^{\lambda(H-\EE[H])}]\leq\exp((\lambda^2/2)\sum_{i=0}^m\kappa_ic_i^2)$. Taking logarithms proves \eqref{equ:independent_input_mgf} for bounded $H$.

For integrable $H$, let $H_M=\max\{-M,\min\{H,M\}\}$. Truncation preserves all the separate Lipschitz constants, so the bound applies to $H_M$. Since $H_M\to H$ almost surely and $\EE[H_M]\to\EE[H]$, Fatou's lemma gives
\begin{align*}
    \EE[e^{\lambda(H-\EE[H])}]\leq\liminf_{M\to\infty}\EE[e^{\lambda(H_M-\EE[H_M])}]\leq\exp\left(\frac{\lambda^2}{2}\sum_{i=0}^m\kappa_ic_i^2\right).
\end{align*}
This completes the proof.
\end{proof}

\begin{lemma}[Exponential moments and moment bounds]
\label{lemma:exponential_moment_bounds}
The following two estimates hold.
\begin{enumerate}[leftmargin=20pt]
    \item If $Z\geq0$ and $\PP(Z>z)\leq e^{-z^2/(2V)}$ for every $z\geq0$, where $V>0$, then
    \begin{align}\label{equ:tail_integration_bound}
        \EE[e^{cZ^2}]\leq\frac1{1-2Vc},\qquad0<c<\frac1{2V}.
    \end{align}
    \item If $W\geq0$, $S>0$, and $\EE[e^{W^2/S^2}]\leq2$, then
    \begin{align}\label{equ:exponential_moment_consequences}
        \PP(W>r)\leq2e^{-r^2/S^2},\qquad(\EE[W^p])^{1/p}\leq S\sqrt p,\qquad r\geq0,\quad p\geq1.
    \end{align}
\end{enumerate}
\end{lemma}

\begin{proof}
For the first part, write $e^{cZ^2}=1+\int_0^{Z^2}ce^{cu}\,du$. The integrand is nonnegative, so Tonelli's theorem gives
\begin{align*}
    \EE[e^{cZ^2}]&=1+\int_0^\infty ce^{cu}\PP(Z^2>u)\,du\\
    &\leq1+c\int_0^\infty e^{-(1/(2V)-c)u}\,du=1+\frac{c}{1/(2V)-c}=\frac1{1-2Vc}.
\end{align*}
The integral is finite because $c<1/(2V)$.

For the second part, Markov's inequality gives $\PP(W>r)\leq e^{-r^2/S^2}\EE[e^{W^2/S^2}]\leq2e^{-r^2/S^2}$. To prove the moment bound, the function $x\mapsto x^{p/2}e^{-x}$ attains its maximum at $x=p/2$. Thus $x^{p/2}\leq(p/(2e))^{p/2}e^x$ for $x\geq0$. Substituting $x=W^2/S^2$ gives
\begin{align*}
    \EE[W^p]\leq S^p\left(\frac{p}{2e}\right)^{p/2}\EE[e^{W^2/S^2}]\leq2S^p\left(\frac{p}{2e}\right)^{p/2}.
\end{align*}
Taking the $p$th root and using $p\geq1$ yields
\begin{align*}
    (\EE[W^p])^{1/p}\leq S\sqrt p\,\frac{2^{1/p}}{\sqrt{2e}}\leq S\sqrt p\,\frac2{\sqrt{2e}}\leq S\sqrt p.
\end{align*}
This completes the proof.
\end{proof}

\section{Additional Experimental Details}
\label{appendix:sec:experiments}

We conducted the experiments on a Linux system, which was equipped with a single Intel(R) Core(TM) i9-14900K CPU with 32 threads and 128 GB of RAM. All methods were implemented using Python. The baseline algorithms used for comparison include: UCB, TS, and TS-SGLD. These implementations follow standard formulations, with TS-SGLD employing one-step SGLD updates per round using minibatches of past rewards. For the TS-SA ablations, let $N$ denote the number of consecutive LMC--SA updates using the same reward window after an arm pull. The main experiments use $N=1$, while values $N>1$ are evaluated only as an extension.

Subsequently, we detail the hyperparameter tuning and ablation experiments for TS-SA. We first conduct standardized experiments to ensure fair comparison across algorithms. We explore the best hyperparameter sets through Bayesian hyperparameter optimization \citep{akiba2019optuna}, which we focus on multi-arm Gaussian bandit problems under different reward gaps $\Delta=\{0.1, 0.5\}$ and different numbers of arms $K=\{10,50\}$. In this setting, we restrict the tuning of TS-SA to only two key parameters: the SGLD step size and the SA coefficient $\alpha$. Similarly, for each baseline algorithm, we tune only a small set of its core hyperparameters to control variance while preserving fairness. This standardized setup allows us to compare different methods under comparable optimization effort.

Moreover, we further perform a broader sensitivity study to investigate the individual contribution of each hyperparameter to overall performance. By varying one parameter at a time while keeping others fixed, we identify key thresholds and sensitivities that govern regret behavior. Our findings reveal which components are critical to tune (e.g., $\Omega$, $\mathcal{B}$, $c_1$, $\alpha$), and which exhibit stable behavior across a wide range of values (e.g., $c_2$, $c_3$), offering practical guidance for TS-SA deployments.

\subsection{Baseline Implementation}\label{appendix:subsec:baseline}

We evaluate our proposed method against several widely used baseline algorithms in MAB literature, including UCB, TS, $\epsilon$-TS, and TS-SGLD. Their implementations are summarized below for completeness.
\begin{itemize}[leftmargin=*,nosep]
  \item \textbf{UCB:} At round $t$, the reward estimation $\btheta_{a,t}$ for arm $a$ is given by 
  $\btheta_{a,t}=\hat{\mu}_a(t)+\tau\sqrt{{2\ln t}/{\mathcal T_a(t)}}$ where $\hat{\mu}_a(t)$ is the empirical mean of the rewards obtained from arm $a$ up to round $t$. We tune $\tau$ to control the estimate for each arm. 
  \item \textbf{TS:} After observing $\mathcal T_a(t)$ rewards from arm $a$ with sample mean $\hat{\mu}_a(t)$, the posterior distribution $p_{a}$ is also a normal distribution $N\left(\frac{{\mu_0}/{\sigma_0^2}+\mathcal T_a(t)\hat{\mu}_a(t)}{{1}/{\sigma_0^2}+\mathcal T_a(t)},\frac{\tau}{{1}/{\sigma_0^2}+\mathcal T_a(t)}\right)$, where $\mu_0$ denotes the prior mean and $\sigma_0^2$ the prior variance. We tune $\tau$ here to control the spread around the posterior mean.
  \item \textbf{$\epsilon$-TS:} With probability $1 - \epsilon$ ($\epsilon\in[0, 1)$), we draw $\btheta_{a,t}$ from the posterior $p_a$ as in TS. The posterior $p_a(t)$ is updated as a normal distribution according to the observed data of arm $a$. With probability $\epsilon$, we choose the arm with the highest expected mean of $p_a(t)$. Here we tune the temperature $\tau$ and the rate of greedy strategy $\epsilon$.  
  \item \textbf{TS-SGLD:} To obtain $\btheta_{a,t}$, we approximate the posterior distribution $p_{a}$ using SGLD. Given the log-posterior density $\log p_a(\btheta|\mathcal B)$ where $\mathcal B \subset \{X_{a, i}\}_{i=1}^{\mathcal T_a(t)}$ is the set of rewards uniformly sampled from rewards of arm $a$ up to round $t$, we adopt the update rule \eqref{eq:euler_update} to collect approximate samples from the last round of each arm and estimate rewards. For a fair comparison, we run SGLD for one step and set $\btheta_{a,t}=\btheta^{(1)}$. We here tune the step size $h$ and the temperature $\tau$.
\end{itemize}

\subsection{Sensitivity Study}

\label{sec:sensitivity}

\begin{figure}[!htbp]
    \centering
    \includegraphics[width=0.98\linewidth]{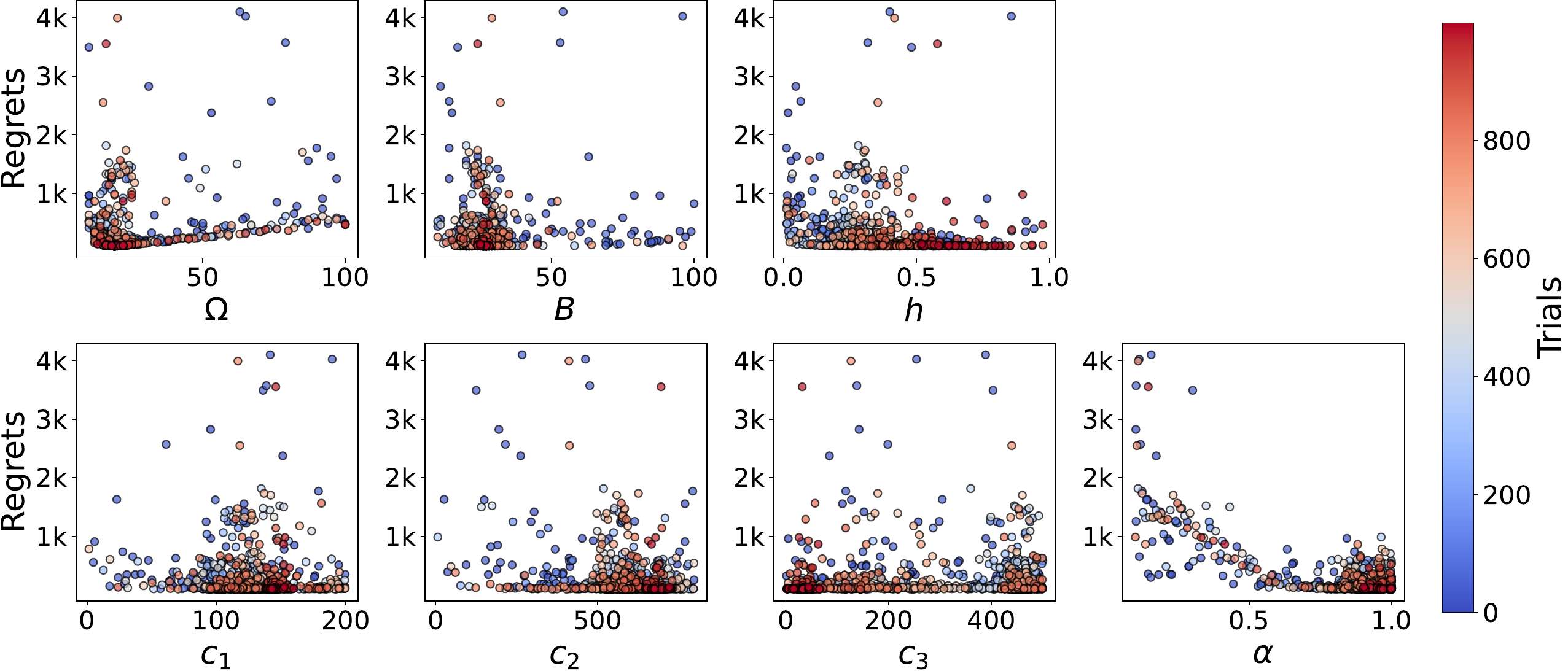}
    \caption{Slice plots of TS-SA hyperparameter tuning using Bayesian hyperparameter optimization, under fixed settings $K=10$, $\Delta=0.5$, $\tau=1.0$, and $N=1$. The $x$-axis denotes the tuning range of a single hyperparameter, while the $y$-axis shows the corresponding cumulative regret averaged over 50 independent runs. The color gradient from blue (early trials) to red (later trials) reflects the optimization trajectory. Regret decreases as Bayesian hyperparameter optimization identifies favorable values for warm-up pulls, batch size, and SA parameters $(c_1, c_2, \alpha)$, while the step size $h$ and offset $c_3$ exhibit relatively weak influence.}
    \label{fig:optuna}
\end{figure}

This subsection presents a sensitivity study on the hyperparameters of the TS-SA algorithm. The goal is to identify which parameters are critical to tune for optimal performance and which have minimal influence. This allows us to fix their values in our main experiments to reduce complexity. The study highlights the robustness of the algorithm to certain parameter choices.

\Cref{fig:optuna} visualizes the hyperparameter optimization process of TS-SA using Bayesian hyperparameter optimization. The x-axis denotes the selected range of each hyperparameter, while the y-axis represents the cumulative regret averaged over 50 independent trials. Each point represents a complete hyperparameter configuration, where all parameters are jointly selected. The color gradient from blue to red reflects the progression of the optimization timeline, where blue indicates early trials and red corresponds to later trials as the Bayesian hyperparameter optimization refines its search based on observed performance.

These plots offer valuable insights into how the optimizer explores and adapts over the course of tuning. Notably, the plots for warm-up pulls, batch size, and stochastic approximation exponent decay parameters ($c_1$, $c_2$, and $\alpha$) exhibit strong concentration toward regions associated with lower regret as the trials evolve from blue to red. This convergence behavior highlights the importance of these hyperparameters and confirms that Bayesian hyperparameter optimization effectively identifies their optimal ranges.

In contrast, the slice plots for the Langevin step size $h$ and offset parameters $c_3$ to adjust SA step size appear relatively flat, with no consistent trend in regret across the optimization process. The absence of directional improvement suggests that these hyperparameters have less impact on performance within the tested range and can be fixed to reduce tuning complexity.

\subsection{Ablation Studies}\label{sec:ablation}

To ensure competitive performance, we tune TS-SA's hyperparameters using Bayesian hyperparameter optimization \citep{akiba2019optuna}. The parameters include the Langevin step size $h$, SA step size parameters $(c_1, c_2, c_3, \alpha)$ via $\gamma = \tfrac{c_1}{c_2 \mathcal T_a(t)^\alpha + c_3}$, temperature $\tau$, mini-batch size $\mathcal{B}$, and warm-start scale $\Omega$.
The tuning proceeds in two stages: a broad search over 1,000 trials to explore wide parameter ranges, followed by a refined 1,000-trial search within narrowed intervals. Tuning is conducted under two settings: fixed $\tau = 1.0$ and tunable $\tau$. Each hyperparameter configuration is evaluated over 50 independent runs to compute average regret. 
For the ablation studies reported below, we use $T=1,000$ rounds and report regret averaged over 100 independent trials.

We further conduct ablation studies to evaluate the sensitivity of individual hyperparameters in our TS-SA algorithm. These experiments utilize a 10-arm Gaussian bandit problem with reward gap $\Delta=0.5$ and temperature $\tau=1.0$.

\textbf{Experimental Setup.} To systematically assess each component's impact, we establish a default configuration where all parameters except the one under investigation remain fixed:
$$
N=1,\ h=0.532,\
c_{1}=144.07,\ c_{2}=677.88,\ c_{3}=40.02,\ \alpha=0.999,\ 
\mathcal B=27,\ \Omega=19.
$$
For evaluating the effects of coefficients ($c_{1}$, $c_{2}$, $c_{3}$) and decay exponent ($\alpha$), we reduce the batch size to $\mathcal{B} = 5$ to better isolate their individual contributions to algorithm performance. For the repeated-update ablation, we decrease the step size to $h=0.05$ and vary $N$ while keeping the remaining settings fixed.

\textbf{Results and Findings.} The results are presented in \Cref{fig:ablation1}, which reveal that both warm-up pulls ($\Omega$) and batch size ($\mathcal{B}$) exhibit a clear threshold behavior. In these ablations, moderate reward windows and warm-up lengths improve performance over very small choices. Increasing either parameter further does not always reduce regret.

Among the parameters of the sampling step (SGLD) and the averaging step (SA), we observe varying degrees of sensitivity:
\begin{itemize}[leftmargin=15pt]
\item The step size ($h$), coefficient ($c_1$), and decay exponent ($\alpha$) significantly influence algorithm performance, requiring careful tuning to achieve optimal regret minimization.
\item In contrast, coefficients $c_2$ and $c_3$ demonstrate minimal impact on overall performance, indicating that TS-SA exhibits robust behavior with respect to these parameters across their tested ranges.
\item The number of inner iterations ($N$) shows a consistently positive relationship with performance, although with diminishing returns at higher values. This suggests that while increasing $N$ improves results, there exists a practical upper limit beyond which additional iterations yield minimal benefits.
\end{itemize}

\begin{figure*}[t]
  \centering
    \begin{subfigure}[b]{0.24\textwidth}
    \centering
    \includegraphics[width=\textwidth]{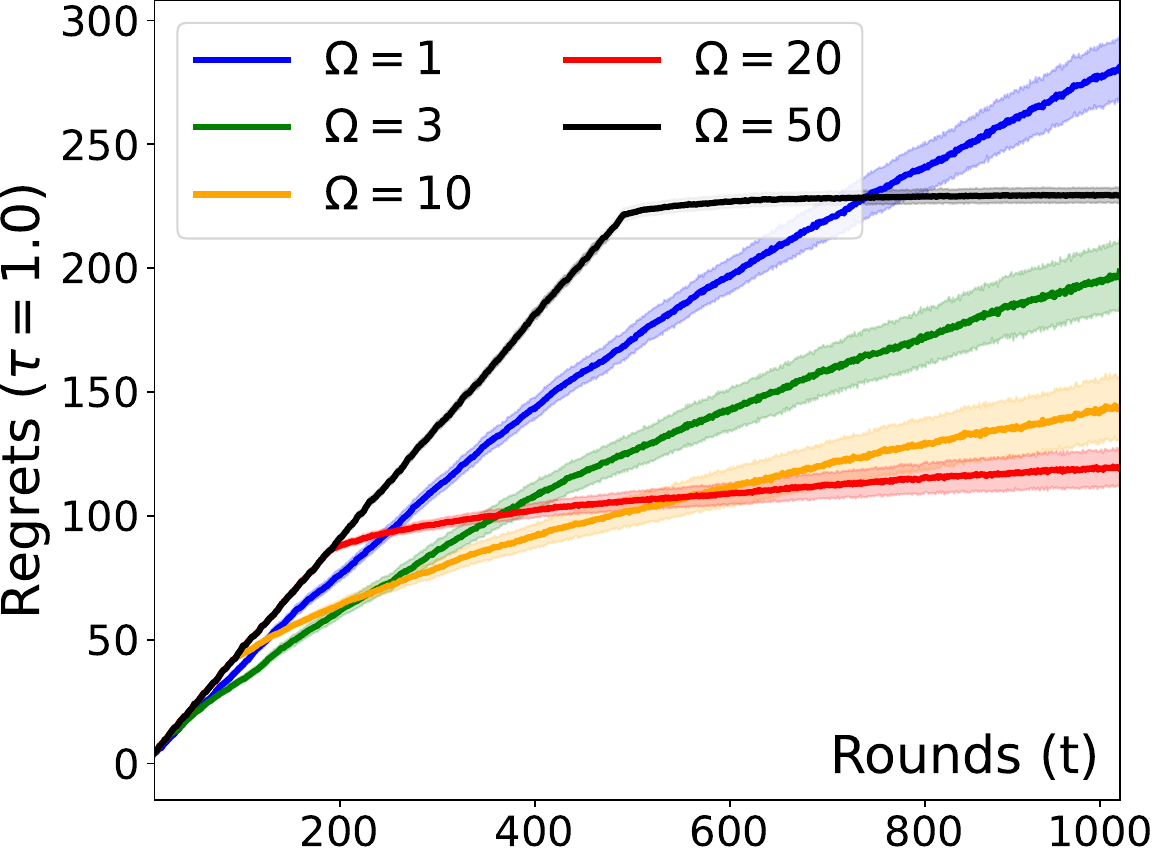}
    \caption{$\Omega$.}
    \label{fig:test:warm}
  \end{subfigure}
    \begin{subfigure}[b]{0.24\textwidth}
    \centering
    \includegraphics[width=\textwidth]{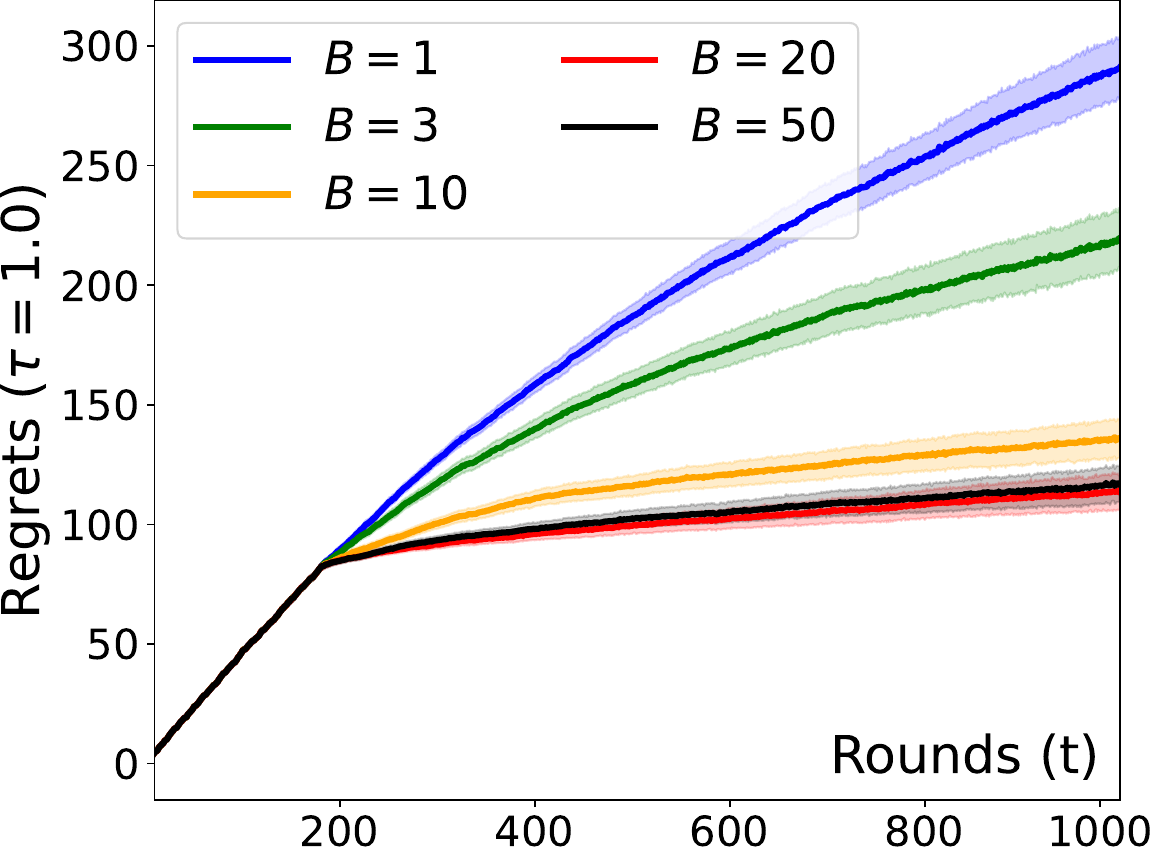}
    \caption{$\mathcal B$.}
    \label{fig:test:batch}
  \end{subfigure}
    \begin{subfigure}[b]{0.24\textwidth}
    \centering
    \includegraphics[width=\textwidth]{figures/ablation/plot_test_niter.pdf}
    \caption{$N$.}
    \label{fig:test:niter}
  \end{subfigure}
    \begin{subfigure}[b]{0.24\textwidth}
    \centering
    \includegraphics[width=\textwidth]{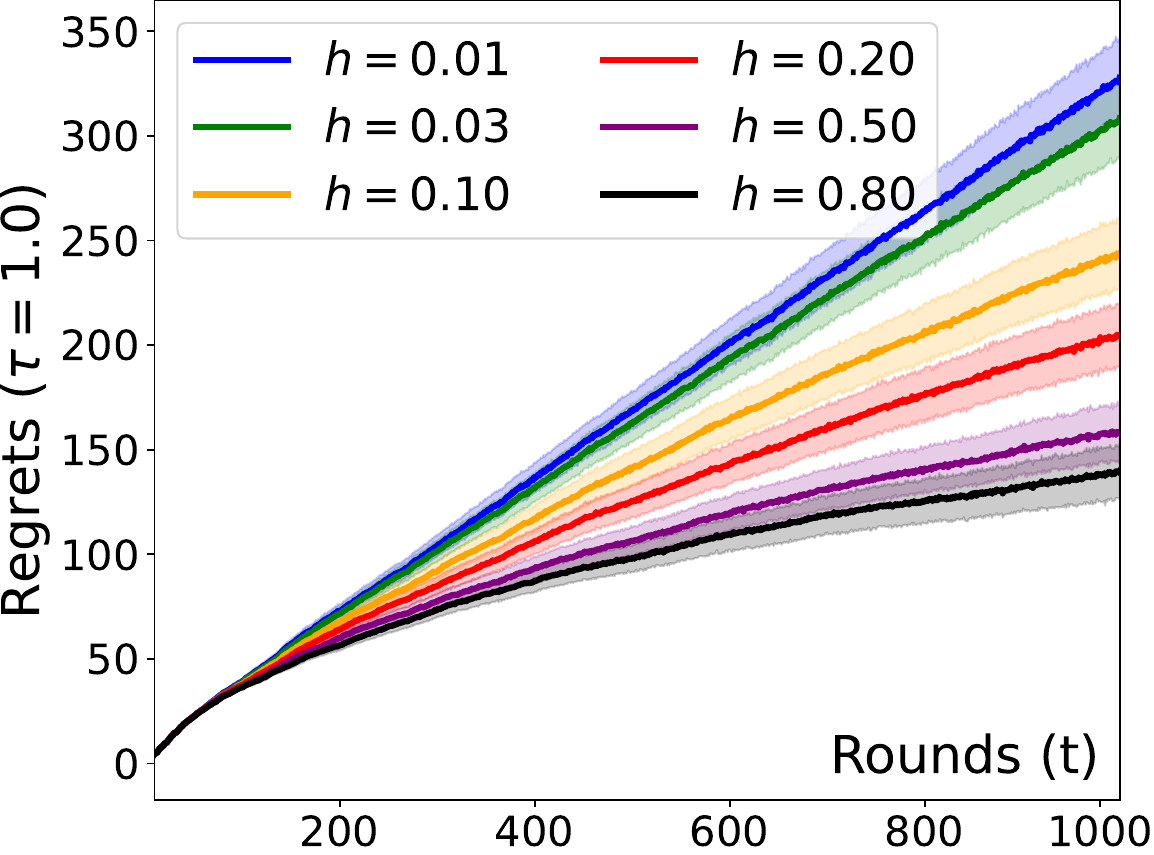}
    \caption{$h$.}
    \label{fig:test:step}
  \end{subfigure}

  \centering
    \begin{subfigure}[b]{0.24\textwidth}
    \centering
    \includegraphics[width=\textwidth]{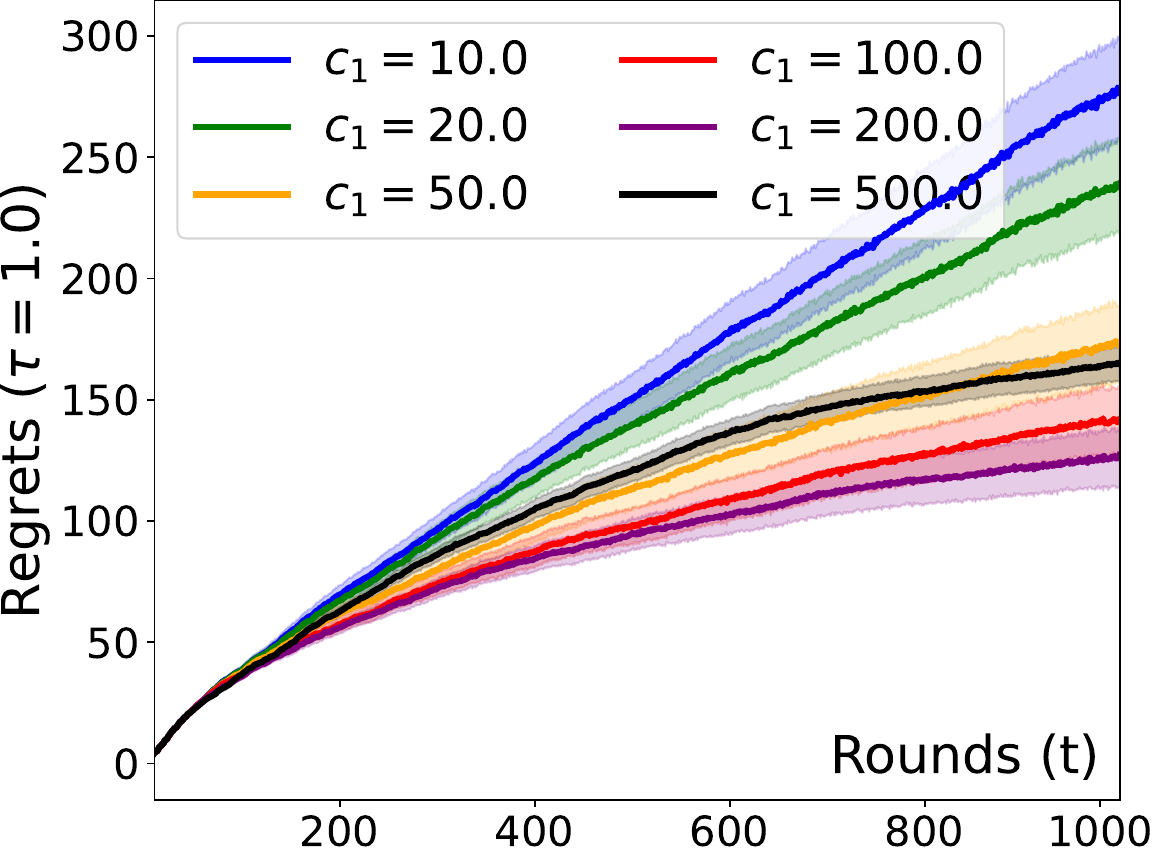}
    \caption{$c_1$.}
    \label{fig:test:c1}
  \end{subfigure}
    \begin{subfigure}[b]{0.24\textwidth}
    \centering
    \includegraphics[width=\textwidth]{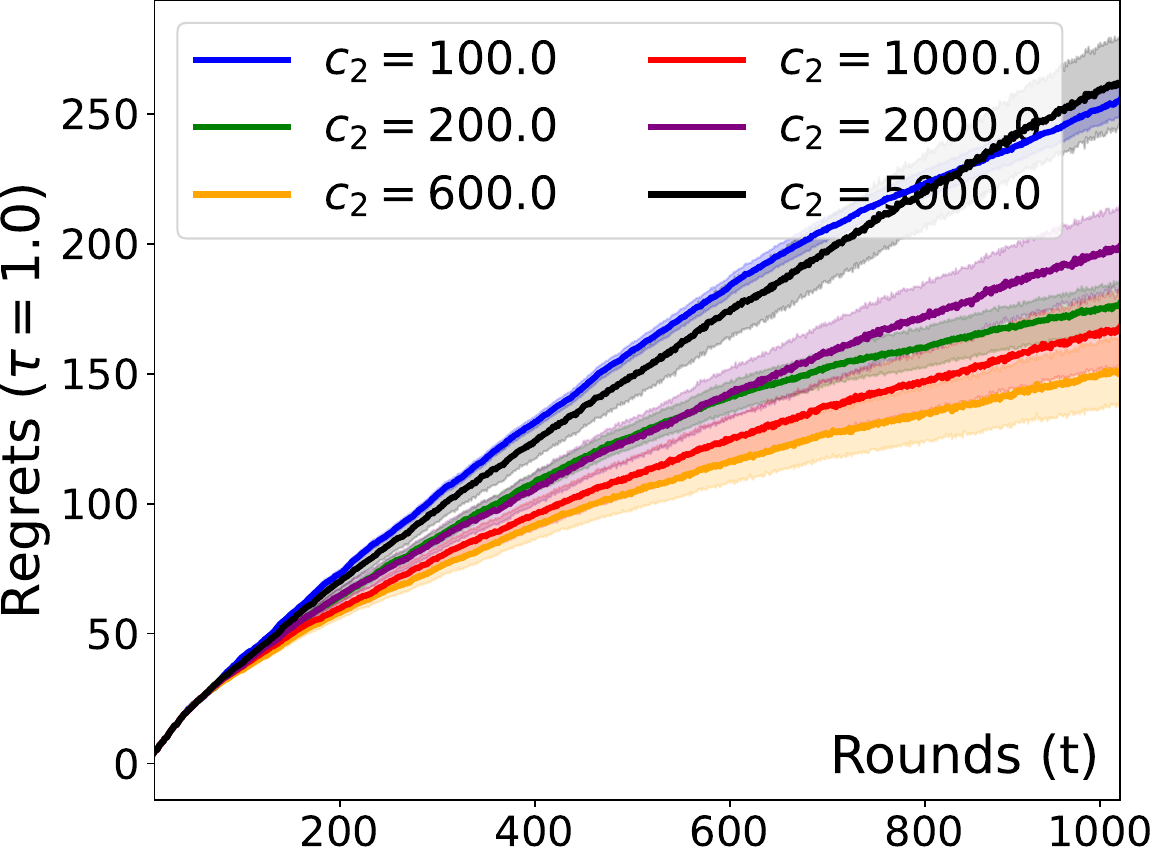}
    \caption{$c_2$.}
    \label{fig:test:c2}
  \end{subfigure}
    \begin{subfigure}[b]{0.24\textwidth}
    \centering
    \includegraphics[width=\textwidth]{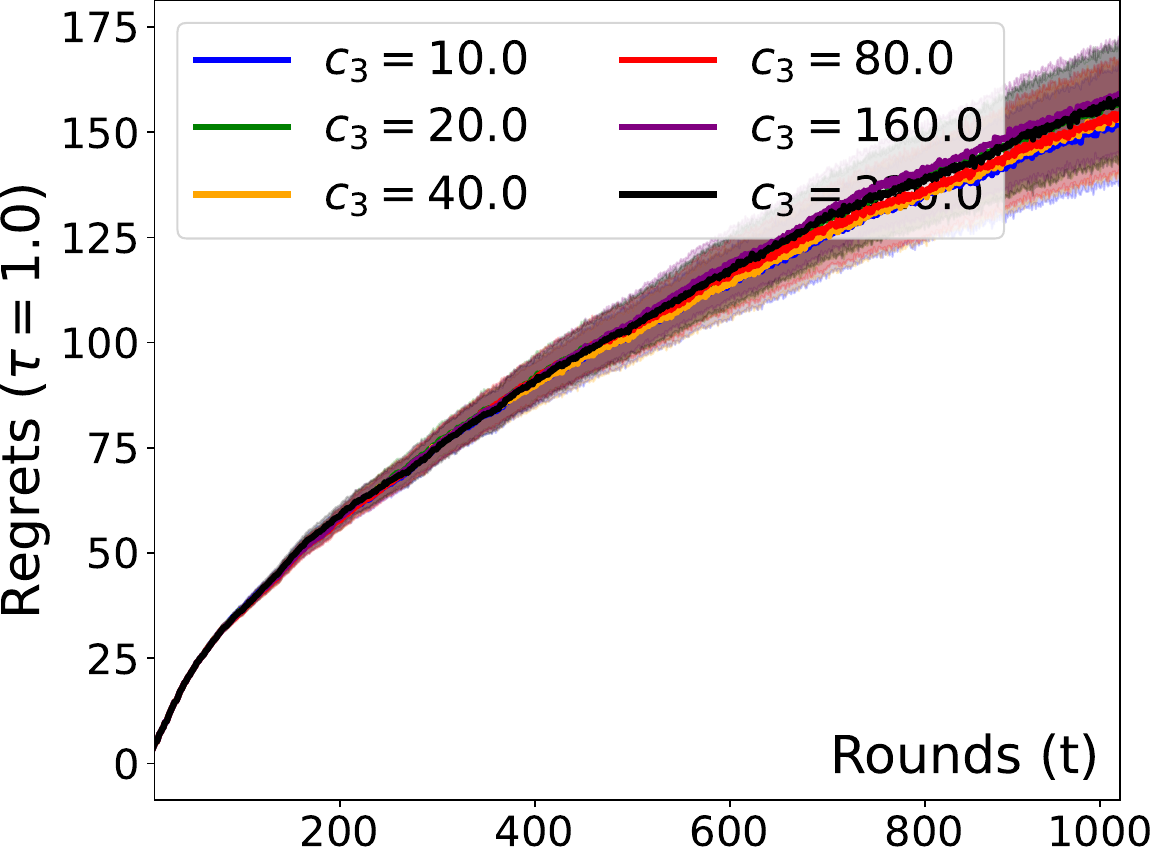}
    \caption{$c_3$.}
    \label{fig:test:c3}
  \end{subfigure}
    \begin{subfigure}[b]{0.24\textwidth}
    \centering
    \includegraphics[width=\textwidth]{figures/ablation/plot_test_alpha.pdf}
    \caption{$\alpha$.}
    \label{fig:test:alpha}
  \end{subfigure}

  \caption{Ablation study of key hyperparameters in TS-SA ($K=10$, $\Delta=0.5$). The x-axis denotes the number of rounds, while the y-axis represents the cumulative regrets over 100 independent trials. Each subfigure plots the cumulative regret as a function of a single parameter while keeping others fixed. (a–b): Warm-up pulls ($\Omega$) and batch size ($\mathcal{B}$) show clear threshold behavior around 20, below which performance degrades significantly. (c): Increasing inner iterations ($N$) improves performance with diminishing returns. (d–h): Sampling-related parameters ($h$, $c_1$, $\alpha$) exhibit notable sensitivity and must be tuned carefully, while $c_2$ and $c_3$ show minor impact, which indicates robustness with respect to these choices.}\label{fig:ablation1}
\end{figure*}

\bibliography{reference}
\bibliographystyle{ims}

\end{document}